\documentclass[pdfa,a4paper,UKenglish,cleveref,autoref,thm-restate]{lipics-v2021}

\usepackage{bbm, nicefrac}
\usepackage{tikz, pgfplots}
\usepgfplotslibrary{fillbetween}
\usetikzlibrary{positioning, matrix}
\usepackage{algorithm,algorithmic}

\DeclareRobustCommand{\bb}[1]{\mathbb{#1}} 
\DeclareRobustCommand{\mr}[1]{\mathrm{#1}}
\DeclareRobustCommand{\beef}{\vphantom{\sum}}
\DeclareRobustCommand{\set}[1]{\left\{#1 \right\}}
\DeclareRobustCommand{\Set}[2][]{\left\{#1 \, \middle|\, #2 \right\}}
\DeclareRobustCommand{\nf}[2]{\nicefrac{#1}{#2}}
\DeclareRobustCommand{\f}[2]{\frac{#1}{#2}}
\DeclareMathOperator{\Var}{Var}
\DeclareRobustCommand{\indep}{\mathbin{\perp \!\!\! \perp}}
\DeclareRobustCommand{\annot}[1]{,\text{{\color{gray} (#1)}}}

\title{A Possibility in Algorithmic Fairness: Can Calibration and Equal Error Rates Be Reconciled?}
\titlerunning{Calibrated Scores and Equal Error Classifiers}

\author{Claire Lazar Reich}{MIT Statistics Center and Department of Economics, Cambridge, MA, USA}{clazar@mit.edu}{}{Supported by the National Science Foundation Graduate Research Fellowship under Grant No. 1122374.}
\author{Suhas Vijaykumar}{MIT Statistics Center and Department of Economics, Cambridge, MA, USA}{suhasv@mit.edu}{}{}

\authorrunning{C.\,L. Reich and S. Vijaykumar}
\Copyright{Claire Lazar Reich and Suhas Vijaykumar}



\relatedversiondetails{Full Version}{https://arxiv.org/abs/2002.07676}

\acknowledgements{Many thanks to Anna Mikusheva, Iv\'{a}n Werning, and David Autor for their valuable advice. We're also deeply grateful for the support of Ben Deaner, Lou Crandall, Pari Sastry, Tom Brennan, Jim Poterba, Rachael Meager, and Frank Schilbach with whom we have had energizing and productive conversations. Thank you to Deborah Plana, Pooya Molavi, Adam Fisch, and Yonadav Shavit for commenting on the manuscript at its advanced stages.}

\EventEditors{Katrina Ligett and Swati Gupta}
\EventNoEds{2}
\EventLongTitle{2nd Symposium on Foundations of Responsible Computing (FORC 2021)}
\EventShortTitle{FORC 2021}
\EventAcronym{FORC}
\EventYear{2021}
\EventDate{June 9--11, 2021}
\EventLocation{Virtual Conference}
\EventLogo{}
\SeriesVolume{192}
\ArticleNo{4}

\nolinenumbers

\begin{document}

\captionsetup[algorithm]{name=Algorithm Stage, labelformat=boxed, position=top}

\maketitle

\begin{abstract}
Decision makers increasingly rely on algorithmic risk scores to determine access to binary treatments including bail, loans, and medical interventions. In these settings, we reconcile two fairness criteria that were previously shown to be in conflict: calibration and error rate equality. In particular, we derive necessary and sufficient conditions for the existence of calibrated scores that yield classifications achieving equal error rates at any given group-blind threshold. We then present an algorithm that searches for the most accurate score subject to both calibration and minimal error rate disparity. Applied to the COMPAS criminal risk assessment tool, we show that our method can eliminate error disparities while maintaining calibration. In a separate application to credit lending, we compare our procedure to the omission of sensitive features and show that it raises both profit and the probability that creditworthy individuals receive loans.
\end{abstract}

\section{Introduction}
\label{intro}

Today's algorithms reach deep into decisions that guide our lives, from loan approvals to medical treatments to foster care placements. 
Making these high-impact decisions fairly is an effort undergoing public scrutiny. 
In one investigation, \textit{ProPublica} showed that an algorithm operating in the U.S. criminal justice system, COMPAS, discriminated against black defendants by misclassifying them as high-risk at significantly higher rates than white defendants \cite{angwin_julia_machine_2016}.
On the other hand, it was later revealed that the same algorithm did satisfy a different form of fairness: calibration of scores for both black and white defendants \cite{feller2016computer}. This meant that on average, a defendant's score reflected the same risk level regardless of race. 

Researchers have sought to explain how a screening algorithm like COMPAS can satisfy one natural notion of fairness but not another, spurring a research agenda to characterize how definitions of algorithmic fairness relate to one another. 
Multiple studies in this literature proved that algorithms face inevitable tradeoffs whenever they predict on groups that have different average outcomes \cite{kleinberg_inherent_2016, chouldechova_fair_2017, berk2018fairness, 10.1145/3097983.3098095, 10.1145/3328526.3329621}. These influential ``impossibility results'' have underscored the need for practitioners to target certain fairness criteria at the expense of others. 

We show that it is in fact possible to reconcile the   
two notions of fairness that gained influence following the COMPAS investigation: calibration and equal error rates. 
In important previous work, these two criteria were proven to be mutually incompatible when both are applied to a \textit{risk score} \cite{kleinberg_inherent_2016, pleiss_fairness_2017} and when both are applied to a \textit{classifier} \cite{chouldechova_fair_2017}.
Naturally these findings were interpreted as evidence that calibration and equal error rates are incompatible altogether  \cite{angwin_julia_bias_2016}. It was therefore speculated that COMPAS's enforcement of score calibration made its error rate imbalances inevitable \cite{chouldechova_fair_2017}.

In contrast, we show that both calibration and equal error rates can be reconciled in COMPAS and in many other real-world settings where protected groups have different mean outcomes.  
We relax the mathematical tension between these two fairness criteria by separately enforcing \textit{calibration on the score} and \textit{equal error rates on the corresponding classifier}. 
In particular, we prove that it is possible to design calibrated scores that yield equal error rate classifications at group-blind cutoffs, and we provide a method to do so with maximal accuracy. Furthermore, we 
develop practical extensions of the method, such as showing how to enforce weaker notions of the equal error rate criterion (like the ``equality of opportunity'' criterion of Hardt et al.~\cite{hardt_equality_2016}) and how to accommodate multiple protected subgroups. 

Our framework and method can be applied to two settings. In the first, we consider the problem of providing risk scores to a profit-maximizing third-party agent, such as a lender, who then uses them to assign binary treatments, such as loan approvals and denials. We illustrate how to construct calibrated scores 
that lead this profit-maximizer to make classifications satisfying equal error rates.
In the second setting, we consider risk assessments like COMPAS that output both scores and classification recommendations, and show that the scores can be made to satisfy calibration while the classification recommendations can be made to satisfy equal error rates.  

This paper supports growing evidence on the complementary relationship between data quality and fairness objectives 
\cite{10.1145/3328526.3329624, 10.1145/2090236.2090255, pmlr-v81-dwork18a, corbett-davies_measure_2018, kleinberg_algorithmic_2018, 10.1145/3328526.3329621}.
In particular, we show that access to sufficiently informative features is required to satisfy our fairness criteria, and that the feasible set of solutions grows with the informativeness of the data. In an empirical credit lending example, we compare our 
method to a commonly practiced strategy of data omission. It yields higher lender profit while also improving access to loans for creditworthy applicants in all groups. 

The results proceed as follows. In Section 2, we prove that it is possible to construct calibrated scores that lead to equal error rate classifications and we precisely characterize when such scores exist. In Section 3, we propose an algorithm that produces the most accurate possible score satisfying the fairness criteria and minimizing the decision-maker's errors. We apply our method in Section 4 to two empirical settings. We first assess its performance in helping a lender screen loan applicants of various educational backgrounds. 
We also apply the method to the COMPAS criminal risk assessment tool, where we show that our procedure can eliminate error rate imbalances in risk classifications while preserving calibration of scores. 

\subsection{Related Work}
Our paper belongs to a body of work that studies the mathematical relationships between various individual and group measures of fairness. Calibration and equal error rates have been formalized and extensively studied in prior work \cite{pleiss_fairness_2017, hardt_equality_2016, kleinberg_inherent_2016, chouldechova_fair_2017}. In particular, Kleinberg et al.~\cite{kleinberg_inherent_2016} and Pleiss et al.~\cite{pleiss_fairness_2017} show that these criteria are incompatible when applied to a risk score and Chouldechova \cite{chouldechova_fair_2017} shows the corresponding result for binary classifiers. 
We consider a natural variation of the problem where we ask whether a calibrated score can, upon being supplied to a rational third-party, lead to equal-error predictions. Surprisingly, we find that the answer is yes. 

Our work also contributes to a recent strand of the literature which studies how 
algorithmic prediction can interact with self-interested decision makers, bridging the classical problem of prediction with the traditionally economic problem of information design \cite{pmlr-v119-perdomo20a, pmlr-v119-shavit20a}. From this perspective, we study the existence of scores that lead to desirable equilibria: those in which the final decision rule is group-blind due to calibration, and the resulting decisions satisfy equal error rates. 

Finally, we believe it is important to emphasize that the two fairness criteria we study do not encompass all notions of fairness. Tradeoffs remain between these criteria and others. For example, enforcing equal error rates requires that the classifications’ positive and negative predictive values will be unequal across groups, meaning that one groups’ scores would carry greater signal to the decision-maker than the others’ \cite{chouldechova_fair_2017}. In addition, equal error rate classifications will generically require changes to the Bayes’ optimal classifications, and enforcing calibration does not diminish this requirement \cite{corbett-davies_measure_2018}.

Decisions for how to prioritize fairness conditions are likely to vary by application going forward.
We hope that by clarifying the precise relationship between two influential criteria, we can facilitate
these decisions, and that in settings where calibration and equal error rates are considered essential,
our algorithm can help yield accurate predictions and fairer outcomes.

\section{Theoretical Results} 
\label{sec:theory}

\subsection{Formal Setting}
\label{sec:definitions}
Let us consider a triple $(Y, X, A)$ on a common probability space $\bb P$, where $Y \in \set{0,1}$ is an outcome variable, $X \in \bb{R}^d$ is a vector of features, and $A \in \set{H,L}$ is a protected attribute differentiating two groups with unequal base rates $\mu_A = \bb{E}[Y|A]$ of the outcome,
\begin{equation}
\mu_L < \mu_H.
\end{equation} 
Our goal is to estimate a score function $\hat{p} \equiv \hat{p}(X, A) \in [0, 1]$ that predicts $Y$ with maximum accuracy subject to the constraints of calibration and equal error rates.  Specifically, we hand $\hat p$ to a decision-maker tasked with selecting classifications $\hat y \in \{0, 1\}$ that minimize their loss function
\begin{equation}
\ell(\hat y, y) = \begin{cases} 0 & y=\hat y  \\ 1 & y > \hat y \\ k &  y < \hat y, \end{cases} \label{lossfunc}
\end{equation}
where $k>0$ is the relative cost of false positive classifications. Note that any loss function that is minimized when $y = \hat y$ is equivalent to $\ell$ after an affine transformation.

Let us suppose the decision-maker might be able to observe group affiliation $A$ in addition to $\hat p$. 
To ensure that classifications are based only on $\hat p$ and not on $A$, we constrain $\hat p$ to satisfy \textit{calibration within groups},
\begin{equation}
\bb{E}[Y|A,\hat p] = \bb{E}[Y|\hat p] = \hat p.  \label{calibration}
\end{equation}
If \eqref{calibration} holds, the decision-maker's expected loss given $\hat p$ and $A$ becomes
\begin{equation}
\bb{E}[\ell(Y, \hat y) | \hat p, A] = \hat p(1-\hat y) + k (1-\hat p) \hat y. \label{calib-loss-characterization}
\end{equation}
This expected loss is minimized with a cutoff decision rule that is independent of group affiliation $A$,
\begin{equation}
\hat y = \mathbbm{1}
\{\hat p \ge \bar p\}, \label{optimal-rule-pbar}
\end{equation}
where the cutoff $\bar p = \nf{k}{(k+1)}$ is fixed by the decision-maker's loss function. 

Our second condition constrains $\hat y$ to satisfy \textit{equal error rates}, ensuring that the classification only depends on the group through the target variable. 
Following the decision rule (\ref{optimal-rule-pbar}), we may write this as
\begin{equation}
(\mathbbm{1}\{\hat p \ge \bar p\} \indep A) \; | \; Y. \label{new-parity}
\end{equation}
Our calibration and equal error rate conditions are summarized by \eqref{calibration} and (\ref{new-parity}), respectively. 

\subsection{Relation to Impossibility Results}
\label{impossibility}
We first introduce a general impossibility result, relate it to previous work, and show where our assumptions diverge to make our proposed criteria satisfiable. The following theorem proves that a \textit{single} algorithmic output $Z$ cannot generally satisfy notions of both calibration and equal error rates. 

\begin{theorem} \label{impossibility-result}
Let $Y, A,$ and $Z$ be random variables satisfying the following three conditions.
\begin{romanenumerate}
\item $(Y \indep A) \;|\; Z,$ \label{gen-cal}
\item $(Z \indep A) \;|\; Y,$  \label{gen-eo}
\item $\mathbb{P}(A=H|Z),\, \mathbb{P}(Y=1|A,Z)\in (0,1)$. \label{pos-def-cond}
\end{romanenumerate}
Then $A$ and $(Z,Y)$ must be independent. 
\end{theorem}
\begin{proof}
Suppose that $(Y, A, Z)$ satisfy \eqref{gen-cal} \eqref{gen-eo} and \eqref{pos-def-cond}. Assumption \eqref{pos-def-cond} implies that the law of $(A,Y,Z)$ is strictly positive. By the Hammersley-Clifford theorem (see e.g.~\cite{10.5555/1592967}), the conditional independence relations are summarized by a graph on $\set{Y, A, Z}$ where every path from $Y$ to $A$ travels through $Z$, and every path from $A$ to $Z$ travels through $Y$. There are only two graphs with this property:
\begin{center}
\begin{tikzpicture}
\tikzstyle{every node}=[draw,circle,fill=white,minimum size=4pt,inner sep=0pt]
 \draw (0,0) node [label=above:$A$]{};
 \draw (1,0) node [label=above:$Z$]{}
	-- (2,0) node [label=above:$Y$]{};

\draw (0,-.8) node [label=above:$A$]{};
\draw (1,-.8) node [label=above:$Z$]{};
\draw (2,-.8) node [label=above:$Y$]{};
\end{tikzpicture}
\end{center}
In neither of these graphs does there exist a path from $A$ to $(Y,Z)$, so we conclude that $A$ and $(Y,Z)$ must be independent for \eqref{gen-cal} \eqref{gen-eo} and \eqref{pos-def-cond} to simultaneously hold.
\end{proof}
Note that when $A$ denotes group affiliation and $Y$ denotes outcomes, \eqref{gen-cal} is a form of calibration and \eqref{gen-eo} is a form of the equal error rate condition. Assumption \eqref{pos-def-cond} is a strong form of predictive uncertainty that is generalized in the appendix. Thus the theorem shows that when there is predictive uncertainty and $Y$ depends on $A$ (i.e.\ when the base rates are unequal), it is impossible for a single $Z$ to satisfy both calibration and equal error rates. For example, letting $Z$ be a classifier recovers the result of Chouldechova that \eqref{gen-cal} equal positive and negative predictive values are unachievable alongside \eqref{gen-eo} equal error rates \cite{chouldechova_fair_2017}. Meanwhile, letting $Z$ be a risk score shows that \eqref{gen-cal} calibration is unachievable alongside \eqref{gen-eo} a condition that implies balance in the positive and negative class, similar to the result of Kleinberg et al.~\cite{kleinberg_inherent_2016}.  

Our own setting bypasses the mathematical impossibility described in Theorem \ref{impossibility-result} by imposing constraints on \textit{two} separate algorithmic outputs rather than one. We require \eqref{gen-cal} calibration from the scores $\hat p$ and \eqref{gen-eo} equal error rates from the resulting classifications $\hat y =\mathbbm{1}\{\hat p \ge \bar p\} $.

\subsection{Necessary and Sufficient Conditions}
\label{sec:nec-suff}
In this section we characterize exactly when there exists a calibrated $\hat p$ that leads to equal error rate classifications $\hat y$ at the cutoff $\bar p$. Our conditions can be easily checked in a given setting, and they are shown to depend on the informativeness of the features $X$.

The graphical framework in this section builds on methods developed by Hardt et al.~\cite{hardt_equality_2016}. All the necessary and sufficient conditions will be illustrated in $\bb{R}^2$, with true positive rates on the vertical axis and false positive rates on the horizontal. The \emph{feasible region} will be the set in $\bb{R}^2$ corresponding to error rates achievable by an equal error rate classifier $\hat y = \mathbbm{1}\{\hat p \ge \bar p\}$ where $\hat p$ is calibrated. 

We first study the entire set corresponding to equal error rate classifiers, without regard to calibration or the decision-maker's cutoff $\bar p$. Then we study the entire set corresponding to classifiers that can be based on the cutoff $\bar p$ applied to calibrated scores, without regard to the equal error rate condition. Finally, we prove that the intersection of these two sets determines feasibility of enforcing both conditions, and we characterize when the intersection is nonempty.    

\subsubsection{Classifiers Satisfying Equal Error Rates}
\label{eq-odds-sec}

We wish to identify the entire set of error rates in $\mathbb{R}^2$ achievable by classifiers with equal error rates. Hardt et al. \cite{hardt_equality_2016} succeeded in doing so, and we review and adapt their results in this subsection. To lay the groundwork for the geometric reasoning to follow, we first denote the group $A$ false positive rate and true positive rate associated with a given classifier $\hat y$ as a point in $\bb{R}^2$,  
\[\alpha(\hat y, A) = \bigg( \bb{P}(\hat y = 1| Y=0, A),\, \bb{P}(\hat y = 1| Y=1, A) \bigg).\]
We may now define the set of achievable error rates in $\mathbb{R}^2$. Let $\mathcal{H}$ be the set of all possibly random classifiers $h(X,A)$.
The set of achievable error rates for group $A$ is 
\begin{equation}
S(A) = \Set[\alpha(\hat y, A)]{\hat y = h(X, A), h \in \mathcal{H}} \subseteq \bb{R}^2,
\end{equation}
and the set of achievable rates for all classifiers satisfying equal error rates is given by $S(L) \cap S(H)$. To better understand this intersection, we characterize $S(A)$ in terms of Receiver Operator Characteristic (ROC) curves following Hardt et al.~\cite{hardt_equality_2016}. By definition, an ROC curve of a given score $p$ traces the true and false positive rates associated with each possible cutoff rule $\mathbbm{1}\{p \ge c\}$ for $c \in [0,1]$. Therefore it contains all points $\alpha(\mathbbm{1}\{p \ge c\},A)$. With these tools in hand, we are ready to characterize the feasible set of rates $S(A)$ for group~$A$. 
\begin{proposition}\label{roc-feas-reg}
Let $p^* = p^*(X,A)$ be the Bayes optimal score satisfying $p^* = \mathbb{E}[Y|X,A]$, i.e., the best score given our data. Then the set of achievable rates $S(A)$ is exactly the convex hull of the union of the group-$A$ ROC curve of the best score $p^*$ and the group-$A$ ROC curve of the worst score $1-p^*$, i.e.~the convex hull of
\begin{equation*}
\begin{split}
& \Set[\beef \alpha(\mathbbm{1}\{p^* \ge c\},A)]{0 \le c \le 1}  \\
& \quad \cup  \Set[\beef (1,1) - \alpha(\mathbbm{1}\{p^* \ge c\},A)]{0 \le c \le 1}.
\end{split}
\end{equation*}
\end{proposition}

Figure \ref{fig1} illustrates typical examples of $S(L)$, $S(H)$, and the intersection $S(L) \cap S(H)$ which represents the rates achievable by equal error rate classifiers.

\begin{figure*}
\begin{minipage}[t]{0.49\textwidth}
\begin{figure}[H]
\centering
\includegraphics[height=2in]{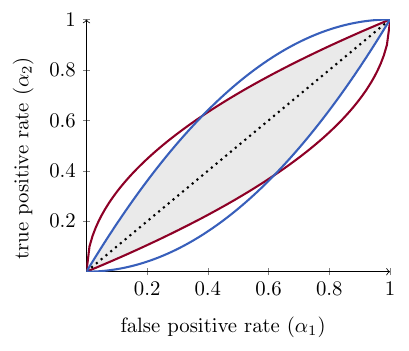}
\caption{\small Achievable equal error rates (shaded). Two pairs of ROC curves form the boundaries of $S(L)$ and $S(H)$. Points in the intersection $S(L) \cap S(H)$ correspond to equal error rate classifiers.}\label{fig1}
\end{figure}
\end{minipage}
\hfill
\begin{minipage}[t]{0.49\textwidth}
\begin{figure}[H]
\centering
\includegraphics[height=2in]{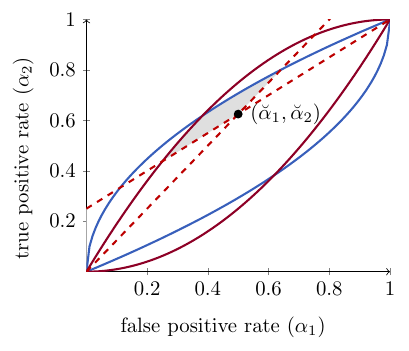}
\caption{\small Achievable equal error rates from calibrated score at cutoff $\bar p$ (shaded). The restrictions \eqref{npv-ppv-thm} correspond to half-spaces above the red dashed lines.}\label{fig2} 
\end{figure}
\end{minipage}
\end{figure*}

\subsubsection{Classifiers Compatible with Calibration }
\label{calibration-sec}

We now put aside the equal error rate constraint and concentrate on identifying the entire set of classifiers that are implementable with the cutoff $\bar p$ applied to some calibrated scores $\hat p$. The set is characterized by the following proposition. 

\begin{proposition} \label{ppv-npv-prelim}
A classifier $\hat y$ can be written as $\hat y = \mathbbm{1}\{\hat p \ge \bar p\}$ for some calibrated $\hat p$ if and only if its group-specific positive predictive values exceed $\bar p$, and its group-specific negative predictive values  exceed $1-\bar p$. In particular, for $ A \in \{ L, H\}$, 
\begin{equation}
\mathbb{P}(Y=1 | \hat y =1, A) \ge \bar p, \quad \mathbb{P}(Y=0 | \hat y =0, A) > 1- \bar p \label{npv-ppv-prelim}.
\end{equation}
\end{proposition}
\begin{proof}
Suppose that $\hat y = \mathbbm{1}\{\hat p \ge \bar p\}$ where $\hat p$ is calibrated. Then $\hat y$ must satisfy the inequalities
\begin{align}
\begin{split}
\bb{P}(Y=1|\hat y = 1, A) 
&= \bb{E}[Y|\hat p \ge \bar p, A] \\
&= \bb{E}[\hat p| \hat p \ge \bar p, A] \ge \bar p, \label{ppv-cond} 
\end{split}\\
\bb{P}(Y=1|\hat y = 0, A) 
&= \bb{E}[\hat p| \hat p < \bar p, A] < \bar p. \label{npv-cond}
\end{align}
Therefore, if $\hat y$ is based on a calibrated score $\hat p$ at cutoff $\bar p$, then it is necessary for the group-specific positive and negative predictive values to exceed $\bar p$ and $(1 - \bar p)$, respectively. 

Conversely, given \emph{any} classifier $\hat y$ that satisfies the inequalities \eqref{ppv-cond} and \eqref{npv-cond}, we can always put
\[\hat p(\hat y, A) = \bb{P}(Y=1| \hat y, A)\] 
to obtain a calibrated score that takes just two possible values per group with the cutoff $\bar p$ guaranteed to be between them. This choice of $\hat p$ thus satisfies $\hat y = \mathbbm{1}\{\hat p \ge \bar p\}$ by construction. 
\end{proof}
As we will see in the following subsection, this result lays the foundation for the necessary and sufficient conditions for the satisfiability of our fairness criteria. 

\subsubsection{The Feasibility Region}
\label{feasible-sec}

Proposition \ref{ppv-npv-prelim} demonstrates that the following are equivalent: 
\begin{romanenumerate}
\item \label{cond-equiv-i} There exists a calibrated score $\hat p$ such that $\hat y = \mathbbm{1}\{\hat p \ge \bar p\}$ satisfies equal error rates.
\item \label{cond-equiv-ii} There exists a classifier $\hat y$ satisfying equal error rates and \eqref{npv-ppv-prelim}.
\end{romanenumerate}

In practice, we propose checking \eqref{cond-equiv-ii} to identify whether \eqref{cond-equiv-i} holds. To do so, we use Bayes' rule to write \eqref{npv-ppv-prelim} as group-specific restrictions on true and false positive rates so that we can consider them in the same space as the equal error rate constraints given by Hardt et al.~\cite{hardt_equality_2016}. The following theorem and the accompanying Figure \ref{fig2} indicate that each restriction \eqref{npv-ppv-prelim} corresponds to a half-space in $\mathbb{R}^2$, and that the feasibile region corresponds to the intersection of those half-spaces with each other and with the equal error rates region $S(L) \cap S(H)$.
\begin{theorem} \label{characterizaton}
Let $\beta_A =\mu_A/(1-\mu_A)$ denote the group-specific odds ratios, with $\beta_L < \beta_H$. Then our fairness criteria are simultaneously satisfiable at cutoff $\bar p $ if and only if there exists $(\alpha_1,\alpha_1) \in S(L) \cap S(H)$ satisfying the two inequalities
\begin{equation}
\f{\alpha_2}{\alpha_1} \ge \f{\bar p}{\beta_L(1 - \bar p)}, \quad 
\f {(1-\alpha_1)}{(1-\alpha_2)} > \f{\beta_H(1-\bar p)}{\bar p}. \label{npv-ppv-thm}
\end{equation}
\end{theorem}
We next provide easily checkable necessary and sufficient conditions for when the feasible region is nonempty. 
\begin{corollary} \label{checkable-cond}
Let $(\breve{\alpha_1}, \breve{\alpha_2})$ denote the point at which the inequalities \eqref{npv-ppv-thm} hold with equality. 
Our fairness criteria are simultaneously satisfiable at cutoff $\bar p$ if and only if any of the following holds: $\breve{\alpha_1}\leq 0$, $\breve{\alpha_1}\geq 1$, or both groups' ROC curves corresponding to $p^*$ lie above $(\breve{\alpha_1}, \breve{\alpha_2})$. Note that $(\breve{\alpha_1}, \breve{\alpha_2})$ are fixed by the group base rates and decision-maker's cutoff $\bar p$, 
\begin{equation}
\breve{\alpha_1} = \frac{\beta_L}{\left(\beta_H-\beta_L \right)} \left( \frac{ \beta_H - \left( 1+\beta_H \right) \bar p }{ \bar p} \right), \qquad
\breve{\alpha_2} = \frac{1}{\left(\beta_H-\beta_L \right)} \left( \frac{\beta_H \left(1 - \bar p \right) - \bar p}{1-\bar p} \right). 
\label{breve}
\end{equation}
\end{corollary}

We note that the feasible region depends on the decision-maker's cutoff $\bar p$, which in turn depends on their relative valuation of false positive and false negative classifications, $k$. In particular, when $k$ is either very large or close to 0, the set of feasible error rates shrinks to include only those corresponding to no positive classifications or no negative classifications.  

Data quality also contributes to the feasibility of enforcing both fairness criteria, as illustrated by Theorem \ref{characterizaton} and Figure \ref{fig2}. Note that the intersection of the half-spaces defined in \eqref{npv-ppv-thm} are fixed by given parameters: $\beta_L$, $\beta_H$, and $\bar p$. Beyond these, what determines the size of the feasible region is the height of the ROC curves.  

Higher ROC curves correspond to more accurate predictions, which can be achieved by including more informative features $X$. This expands the region $S(H) \cap S(L)$ and thus always weakens the constraints dictating whether equal error rates and calibration are compatible in a given setting. Therefore, increasing the quality of data that an algorithm can access promotes our notions of fairness, whereas removing data compromises them. 

\section{A Loss-Minimizing Algorithm} 
\label{sec:algo}

After checking that our fairness criteria are feasible in a given setting, a natural next step is to search for the constrained optimal solution, i.e.\ to identify the most accurate score $\hat{p}$ that minimizes the decision-maker's loss subject to our fairness constraints. Our strategy is to first estimate the most accurate score $p^* = \bb{E}[Y|X,A]$ without regard to fairness, and then to transform the estimate in two separate stages. \textcolor{black}{First, we identify the error rates that minimize loss subject to the fairness conditions (Section~3.1)}. Second, we identify the MSE-minimizing calibrated scores $\hat p$ that gives rise to those error rates at the decision-maker's cutoff $\bar p$ (Section~3.2). Lastly, we lay out extensions of the algorithm that can accommodate practical use cases (Section 3.3). 

\subsection{Stage 1: Error Rate Optimization}
\label{sec:err-rates-alg}

The first stage of the algorithm identifies feasible error rates that minimize the decision maker's loss. 

Let $R$ denote the set of points $(\alpha_1, \alpha_2)$ in the feasible region, i.e.~the pairs of error rates in $S(H) \cap S(L) $ that satisfy \eqref{npv-ppv-thm}. 
Note that $R$ is necessarily convex, as it is the intersection of four convex regions: $S(H)$, $S(L)$, and the half-spaces defined in \eqref{npv-ppv-thm}. Moreover, according to the decision-maker's loss function, a classifier corresponding to error rates $(\alpha_1,\alpha_2)$ obtains expected loss \begin{equation}
\ell(\alpha_1, \alpha_2) \equiv k\alpha_1(1-\bb E[Y]) + (1-\alpha_2)\bb E[Y] \label{feas}.
\end{equation}
Thus, straightforward convex optimization will identify the error rates that minimize the linear function $\ell$ over $(\alpha_1, \alpha_2) \in R$.
The optimal error rates identified, $z^* = (\alpha_1^*,\alpha_2^*)$, will be on the upper-left boundary of the feasible region in Figure \ref{fig2}, with the precise point determined by the decision-maker's relative preference $k$ over false positive and false negative classifications. 

\begin{algorithm}[H]
\caption{Find loss-minimizing feasible error rates.}
\label{alg:error_alg_feas}  
\begin{algorithmic}
\STATE {\bfseries Input:} Raw scores $\{p_i^*\}$, labels $\{Y_i\}$, group identities $\{A_i\}$, base rates $\mu_A$, cutoff $\bar p$, loss parameter $k$.
\STATE {\bfseries Step 1:} Define convex feasible region $R$ by taking intersection of rates $(\alpha_1,\alpha_2)$ in $S(L) \cap S(H)$ that satisfy (\ref{npv-ppv-thm}). To compute $S(L) \cap S(H)$, use $\{p_i^*\}$ to determine each group's ROC curves.
\IF{$R$ is empty}
\STATE {\bfseries Output:} No feasible solution.
\ENDIF
\STATE {\bfseries Step 2:} Minimize loss function (\ref{feas}) over $(\alpha_1, \alpha_2) \in R$. \STATE {\bfseries Output:} Optimal target rates $(\alpha_1^*, \alpha_2^*)$ from Step 2.
\end{algorithmic}
\end{algorithm}

\begin{remark} 
The sets $S(L)$ and $S(H)$ correspond to the Bayes optimal score $p^* = \mathbb{E}[Y|X,A]$, which needs to be estimated in practice. Given an estimated score $p$, we propose using a holdout sample to first calibrate $p$ and then perform our algorithm. The resulting scores will satisfy the fairness criteria approximately by a law-of-large-numbers argument, where the fidelity is determined solely by the size of the holdout sample (see e.g.~\cite{wegkamp2003}). 
\end{remark}
\subsection{Stage 2: Risk Score Optimization}
\label{sec:alg-scores}

\textcolor{black}{Once a feasible set of error rates is chosen,} the decision-maker's expected loss is determined. However, multiple choices of calibrated scores may achieve those target rates at the cutoff $\bar p$, and we expect that in practice, decision-makers would prefer more accurate scores. This section thus describes a method to recover the MSE-minimizing score $\hat p$ that implements the target rates \textcolor{black}{$z^*$} by solving a constrained optimal transport problem \cite{peyre2019computational}. 

We base the method on the finding that the best $\hat p$ satisfying the fairness criteria is recoverable through post-processing the Bayes optimal score $p^* = \mathbb{E}[Y|X,A]$. We include a proof for this in the appendix, following a similar argument of Hardt et al.~\cite{hardt_equality_2016}. In the appendix we also discuss how our procedure can be thought of as finding the smallest \textit{mean-preserving contraction} of $p^*$ that yields the targeted error rates. Readers will note that the post-processing procedure requires some randomization of input scores. We explore the effects of the randomization empirically in our online appendix \cite{reich2020possibility}, and meanwhile highlight that our algorithm's accuracy objective limits the extent to which scores $p^*$ change. 

Our method defines one linear program per group $A$ and seeks the most accurate $\hat p_A$ that yields  error rates at the cutoff $\bar p$ given by  \textcolor{black}{
\[\alpha(\mathbbm{1}\{\hat p_A \ge \bar p\}, A) = z^* = (\alpha^*_{1},\alpha^*_{2}).\]}For the remainder of the section, we simplify notation by suppressing $A$ subscripts and note that the procedure is performed once for each group $A \in \{H, L\}$.  

Our approach will involve a transformation kernel, or \emph{transport map}, that maps the distribution of the most accurate estimate of $p^*$ to the distribution of our post-processed $\hat p$. 
We assume for simplicity that the $p^*$ estimate has already been calibrated, and that it is discrete (which we justify in the appendix). In particular, $p^*$ takes $N$ ordered values $p = (p_1, p_2, \ldots, p_N)$, each with probability mass given by $s = (s_{1}, s_{2}, \ldots, s_{N})$ where $\sum_i s_{i} =1$. Furthermore, we will denote the post-processed $\hat p$ as taking those same discrete values $p$ but with different probability masses that we seek to optimize, $f = (f_{1}, f_{2}, \ldots, f_{N})$.

We call $T$ the matrix that maps probability masses from the discrete distribution of $p^*$ to that of $\hat p$. In particular, with probability $T_{ij}$, the kernel will map an individual with score $p_{i}$ to the output score $p_{j}$. Therefore, the probability distribution of $\hat p$ will be determined by 
\begin{equation} 
T^\top s = f. \label{alg:tsf}
\end{equation}
In order to produce probability distributions, $T$ must be right-stochastic: elements must take values between 0 and 1, and each row should sum to 1.
\textcolor{black}{
\begin{equation}0 \le T_{ij} \le 1 \text{ and } \sum_{k=1}^N T_{ik} = 1 \quad \forall\,i,j \in \{1,\ldots N\}. \label{alg:prob} \end{equation}}
According to our fairness criteria, we further constrain $T$. 
To ensure that $\hat p$ will be calibrated, we need the outcome of individuals assigned score $f_i$ to satisfy $Y=1$ with probability $p_i$. Given our assumption that $p^*$ has itself been calibrated, this reduces to
\begin{equation}
\sum_{i = 1}^N T_{ij} p_i s_{i}   = p_j f_j \quad \forall\, j \in \{1,\ldots,N\}.  \label{alg:calib}
\end{equation}
The targeted false- and true-positive rates $(\alpha_{1}^*, \alpha_{2}^*)$ derived in Section 3.1 similarly require:  
\begin{equation}
\begin{split}
\sum_{j=1}^N \sum_{i=1}^N  T_{ij} p_i s_{i} \left(\mathbbm{1}\{p_j \ge \bar p\} - \alpha_{2}^* \right) 
&= 0, \\
\sum_{j=1}^N \sum_{i=1}^N T_{ij} (1-p_i) s_{i} \left(\mathbbm{1}\{p_j \ge \bar p\} - \alpha_{1}^* \right)
&= 0.\end{split} \label{alg:fpr-tpr}
\end{equation}
Finally, we formulate an objective. Note that \textcolor{black}{
the mean-squared error of $\hat p$ satisfies the bias-variance decomposition
\[\bb{E}[(\hat p - Y)^2] = \bb{E}[(\hat p - \bb{E}[Y|X,A])^2] + \bb{E}[(Y- \bb{E}[Y|X,A])^2],\] 
and thus the $\hat{p}$ that minimizes the left hand side is obtained by minimizing the first term on the right hand side. In particular,} 
if the input score $p^*$ is $\bb{E}[Y|X,A]$, then the post-processed score that minimizes mean-squared error will also minimize 
\begin{equation}
\bb{E}[(\hat p - p^*)^2] = \sum_{i=1}^N\sum_{j=1}^N T_{ij}   (p_i -  p_j)^2 s_{i}. \label{surrogate}
\end{equation}
Furthermore, even if $p^*$ is not exactly equal to $\bb{E}[Y|X,A]$, the triangle inequality in $L^2(\bb P)$ implies 
\[\bb{E}[(\hat p - Y)^2]^{\frac 1 2} \le \bb{E}[(p^* - Y)^2]^{\frac 1 2} + \bb{E}[(\hat p - p^*)^2]^{\frac 1 2}.\]
Thus, by minimizing the objective \eqref{surrogate} we can effectively control the additional error due to post-processing. Combining this with the above constraints yields a straightforward linear program.

\begin{algorithm}[H]
\caption{For each group, find calibrated scores achieving target rates.}
\label{alg:scores}
\begin{algorithmic}
\STATE {\bfseries Input:} Raw scores $\{p_i^*\}$, number of bins $N$, target error rates from stage 1 of algorithm $(\alpha_{1}^*,\alpha_{2}^*)$, cutoff $\bar p$.
\STATE {\bfseries Step 1:} Produce discrete score approximation of $p^*$: label $N$ ordered values $(p_1, p_2, \ldots, p_N)$ with masses $(s_1, s_2, \ldots, s_N)$.
\STATE {\bfseries Step 2:} Find score transformation kernel $T$ that minimizes (\ref{surrogate}) subject to the constraints (\ref{alg:tsf}), (\ref{alg:prob}), (\ref{alg:calib}) and (\ref{alg:fpr-tpr}). \STATE {\bfseries Step 3:} Map each individual with given raw score to a new post-processed score, based on probabilities given by kernel $T$.
\STATE {\bfseries Output:}  Scores $\hat p$ from Step 3. By design these satisfy calibration and yield  error rates $(\alpha_{1}^*,\alpha_{2}^*)$ at cutoff $\bar p$.
\end{algorithmic}
\end{algorithm}

\subsection{Available Extensions}

Our procedure can be modified to handle additional use cases. We can flexibly trade off the fairness and accuracy objectives, minimize error disparities rather than eliminate them when the feasible region is empty, accommodate a setting where the decision-maker's cutoff $\bar p$ is estimated with error, and apply the procedure to more than two groups. 

\subsubsection{Relaxing the fairness criteria}

An alternative formulation of our algorithm can accommodate multiple cases encountered in practice. By modifying Stage 1 to include a weighted error-rate penalty, users can flexibly trade off the fairness and accuracy objectives, minimize error disparities rather than eliminate them when the feasible region $R$ is empty, and enforce just one error constraint as in the ``equality of opportunity'' criterion of Hardt et al.~\cite{hardt_equality_2016}. In general, the more flexible procedure will output group-specific optimal error rates: $z_L^*$ and $z_H^*$. These group-specific targets are then inputted into Stage 2 which is otherwise unchanged. 

To modify Stage 1, first we define a broader domain for the algorithm to search over in place of $R$. It contains all the error rates implementable by a calibrated score at the decision-maker's cutoff, according to the inequalities (\ref{npv-ppv-thm}), without regard to equal error rates. The domain is $R(H) \times R(L)$ where
\begin{equation} R(A) = \Set[(\alpha_1,\alpha_2) \in S(A)] { \frac{1-\alpha_2}{1-\alpha_1} < \frac{\bar p/\beta_A}{(1-\bar p)} \le \frac{\alpha_2}{\alpha_1} }.\end{equation}
(Note that this is guaranteed to be nonempty, as it contains the error rates of the classifier $\mathbbm{1}\{p^*\ge\bar p\}$.)
We also replace the loss function (\ref{feas}) with a generalized version that includes both the decision-maker's expected loss from the error rates as well as the groups' rate disparities. The new loss function is
\begin{equation}
\gamma \ell(z_L) + (1-\gamma)\ell(z_H) +  (z_L - z_H)^\top \Lambda (z_L-z_H)
\label{dual}
\end{equation}
where $\ell(z_A)$ is the decision-maker's  expected loss $k \alpha_{1A}(1-\bb{E}[Y|A]) +  (1-\alpha_{2A}) \bb{E}[Y|A]$ and $\gamma$ is the fraction of individuals in group $L$. Meanwhile, $\Lambda$ is a positive semidefinite matrix that provides the flexibility of varying the enforcement of minimal error rate differences. For example, taking $\Lambda = \lambda I$ for arbitrarily large $\lambda$ recovers the equal error rate solution when the feasible region $R$ is nonempty, and otherwise outputs the solution that minimizes error rate disparities. Meanwhile a small choice of $\lambda$ places relatively more weight on accuracy. 

Alternatively, $\Lambda$ could be chosen so that differences in the true and false positive rates are weighted differently. For example, we can achieve equal true positive rates and allow false positive rates to vary \cite{hardt_equality_2016} by letting $\Lambda(2,2)$ be large and assigning 0 to all other entries in $\Lambda$. 

As a result of the flexible procedure, group-specific error rates $z_L^*$ and $z_H^*$ are identified to minimize the generalized loss function (\ref{dual}). The second stage of the algorithm can then be applied to identify a calibrated score that yields those target rates. 

\subsubsection{Accommodating an interval of possible \texorpdfstring{\boldmath $k$}{k} or \texorpdfstring{\boldmath $\bar p$}{p ̄}}

In settings where the exact $\bar p$ is unknown or not fixed, users can adapt our algorithm to function for any cutoff in an interval $(\bar p - \epsilon, \bar p +\epsilon)$. It can be tailored to produce scores $\hat p$ that are either below $\bar p - \epsilon$ or above $\bar p + \epsilon$, so that any cutoff applied within the interval would execute the same classifications. 

In particular, we propose a couple modifications to generalize our algorithm to this setting. We wish for anyone receiving scores above $\bar p + \epsilon$ to be classified as $\hat y =1$ and anyone receiving scores below $\bar p - \epsilon$ as $\hat y =0$. Following the reasoning in Proposition \ref{ppv-npv-prelim}, for such a score to be calibrated, the associated PPV should exceed $\bar p + \epsilon$ and the NPV should exceed $1-(\bar p - \epsilon)$. Therefore, the feasible region previously defined in Theorem \ref{characterizaton} by (\ref{npv-ppv-thm}) is now defined by the points $(\alpha_1,\alpha_1) \in S(L) \cap S(H)$ that satisfy
\begin{align}
\f{\alpha_2}{\alpha_1} \ge \f{\bar p + \epsilon}{\beta_L(1 - (\bar p + \epsilon))}, \quad 
\f {(1-\alpha_1)}{(1-\alpha_2)}> \f{\beta_H(1-(\bar p-\epsilon))}{\bar p - \epsilon}. 
\end{align}

This feasible region is used in Stage 1. In Stage 2, we add another constraint to specify that no post-processed scores be assigned values inside the interval of possible cutoffs:
\begin{equation}
T_{ik}=0  \quad \forall k \text{ such that } p_k \in (\bar p - \epsilon, \bar p +\epsilon).
\end{equation}

The rest of the procedure remains unchanged. The cost of the added flexibility is a tighter feasible region and higher MSE of the final score.  

\subsubsection{Satisfying the criteria for more than two protected groups}

The algorithm can be modified to satisfy the fairness criteria for multiple groups, across multiple identifiers. First define each group as a unique combination of protected features. Then, the feasible set of error rates is given by the intersection of each $S(A)$ with the points satisfying the inequalities (\ref{npv-ppv-thm}) where $H$ is the highest-mean group and $L$ is the lowest-mean group. Stage 1 of the algorithm proceeds to find the optimal set of error rates in that feasible region. Stage 2 proceeds as usual, implementing a separate program for each group.  

\section{Empirical Results}

\begin{figure*}[t!]
\centering
\begin{subfigure}{0.32\textwidth}     
\centering
\includegraphics[height=1.75in]{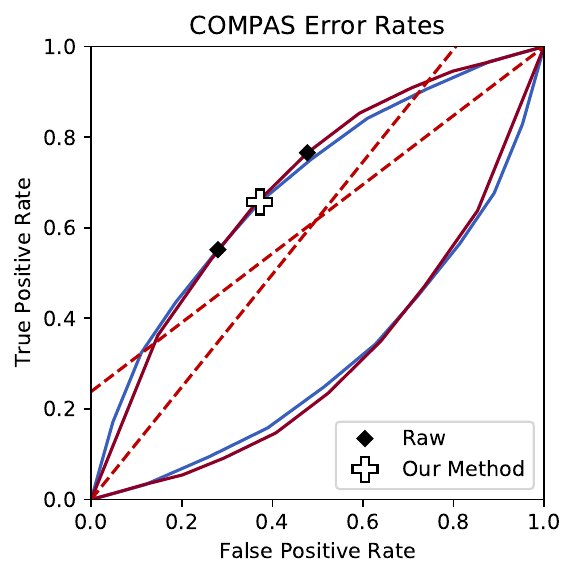}
\subcaption{}\label{3b}
\end{subfigure}\begin{subfigure}{0.32\textwidth}
\centering
\includegraphics[height=1.75in]{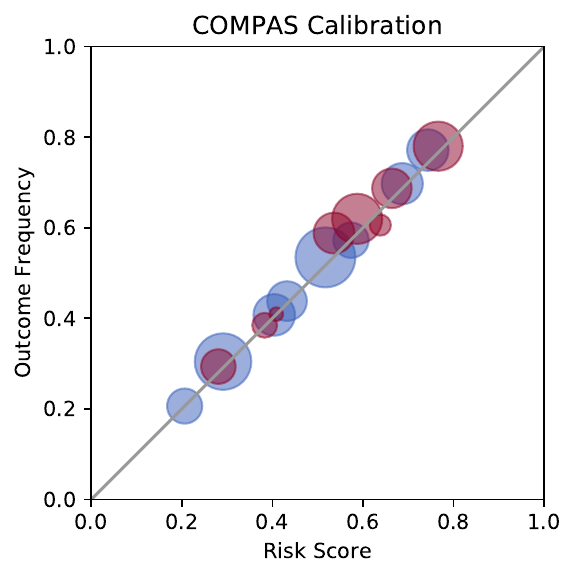}
\subcaption{}\label{3c}
\end{subfigure}
\begin{subfigure}{0.32\textwidth}
\centering
\includegraphics[height=1.75in]{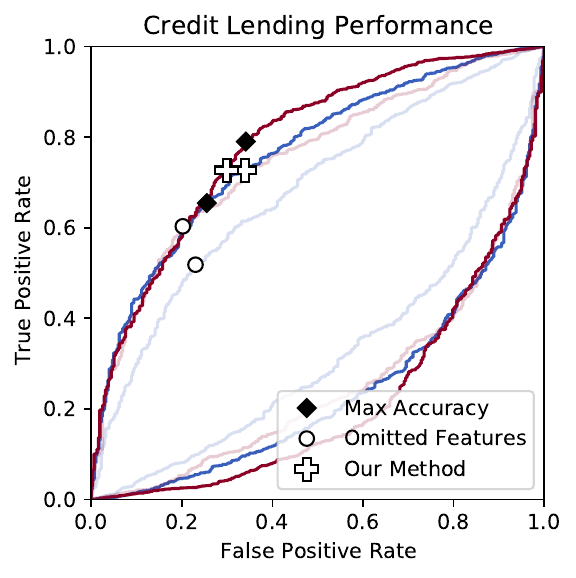}
\subcaption{}\label{3a}
\end{subfigure}
\caption{\small Evaluating algorithm performance. In each figure, maroon represents the high-mean group while blue represents the low-mean group. Panels (a) and (b) correspond to  
the criminal justice application, showing respectively that we can eliminate error rate disparities and maintain score calibration in COMPAS. Note that we define a true positive classification as correctly identifying someone who would not reoffend.
Panel (c) covers the  credit lending application, illustrating the empirical ROC curves from the rich feature set (opaque) and the limited feature set (translucent). Compared to a data omission strategy, our method raises the probability that creditworthy individuals from all education groups access loans.}
\end{figure*}

Let us take our procedure to data. In the first application, 
we post-process real COMPAS scores to demonstrate that risk assessments can be designed to output both calibrated risk scores as well as binary risk summaries satisfying equal error rates. 
Afterwards, 
we design a risk score to aid a  
lender's classification task to authorize loans, showing that it outperforms a common alternative strategy based on the omission of sensitive features. 
For interested readers, extensive detail about each application is presented in our online appendix \cite{reich2020possibility}. 

\pagebreak
\subsection{Predicting criminal recidivism}
Our procedure can design risk assessments that output both calibrated scores as well as binary ``high'' or ``low'' risk summaries satisfying equal error rates. We illustrate this in our first application, where we modify real criminal justice risk scores from COMPAS. As noted earlier, a \textit{ProPublica} investigation showed that current COMPAS scores yield error imbalances across race, although they satisfy predictive parity overall \cite{angwin_julia_machine_2016,angwin_julia_bias_2016}. 

To check whether we can correct COMPAS error imbalances without sacrificing score calibration, we applied our post-processing technique to Broward County risk scores made public by \textit{ProPublica} \cite{propublica_compas_data}. We define the outcome of interest as recidivism within two years, and we convert existing COMPAS scores that range from $[1, 10]$ to probabilities in $[0,1]$.
We define the classification cutoff as the minimum score of defendants classified as ``high risk'' in COMPAS, according to \textit{ProPublica}'s influential analysis \cite{larson_how_nodate}. This corresponds to a cutoff of $\bar p = 0.54$ and loss parameter $k = 1.17$. 

We compute the feasible region of achievable error rates according to Stage 1 of our algorithm and identify the loss-minimizing pair, as depicted in Figure \ref{3b}. Then, we use Stage 2 to post-process the COMPAS scores to achieve new calibrated scores yielding that optimal pair of error rates. The calibration of our scores is depicted in Figure \ref{3c}, where we group together by race defendants with the same post-processed scores and show that their corresponding recidivism outcomes lie on the main diagonal. Overall, our procedure eliminates the reported error disparities across racial groups (Figure \ref{3b}) while also preserving calibration (Figure \ref{3c}).  

\subsection{Predicting loan repayment }

We next present an example of designing a risk score to inform a credit lender’s approvals of loan applicants. Our goal is to deliver to the lender calibrated scores for applications from two groups\textemdash  one highly educated ($H$) and another less educated ($L$)\textemdash while ensuring that they yield classifications with equal group TPRs at the lender cutoff. That way, we know that qualified applicants will have the same probability of receiving a loan regardless of their education level. We suppose the lender in question views defaulting as highly costly and only authorizes loans to individuals with calibrated scores greater than $\approx .9$, corresponding to loss parameter $k=10$. 

We simulate this scenario by applying our algorithm to the Survey of Income and Program Participation (SIPP), a nationally-representative survey of the civilian population spanning multiple years \cite{datacite}. We select as our outcome the ability to pay rent, mortgage, and utilities in 2016, and predict that outcome using survey responses from two years prior. We label individuals with at most a secondary school education as $L$ and those with higher education as $H$.  

\enlargethispage{1.8\baselineskip}
The full dataset contains over 1,800 features spanning \textcolor{black}{detailed} financial variables (including work history, assets, and debts), as well as sensitive features (including demographic information). We apply our \textcolor{black}{algorithm} to the full feature set and derive calibrated scores that yield equal \textcolor{black}{TPRs at the lender's cutoff}, using our algorithm extension that allows FPRs to vary. Then, we compare its performance to two accuracy-maximizing procedures: one based on the full feature set, and another commonly-practiced approach based on the omission of sensitive features. 
The results are summarized numerically in Table \ref{t1} and graphically in Figure \ref{3a}. Compared to prediction on all features and no post-processing, our algorithm raises the TPR of $L$ and lowers that of $H$, while raising lender loss. Meanwhile, compared to the commonly used data omission strategy, our algorithm raises the probabilities that creditworthy applicants from \textit{both} education groups are granted loans, and lowers loss for the lender.  
\pagebreak

\begin{table*}
\caption{\small Application to credit lending. 
Row [1] is based on raw scores.  
Row [2] summarizes the classifier that minimizes lender loss subject to equal \textcolor{black}{true positive rates, given by the equal opportunity algorithm in Hardt et al. (2016). Row [3]  
summarizes our algorithm, which produces a calibrated score corresponding to equal \textcolor{black}{true positive rate} classifications}; since it retrieves the same error rates as row [2], we see there is no added loss from enforcing score calibration. 
Row~[4] summarizes the scores from the \textcolor{black}{alternative procedure that omits sensitive features}, displaying greater loss for the lender, lower true positive rates for both groups, and substantial error disparities across groups.}
\label{t1}
\centering
\begin{tabularx}{.95\textwidth}{ll|ccccr}
& Algorithmic Target & Lender Loss &  TPR (H/L)     &    FPR (H/L)   & Score MSE \\ \hline \hline 
&\emph{Trained on all features} &      &               &               & \\
$[1]$&Accuracy Maximizing& .517      & (.795/.661)   &  (.341/.255)  & .072               \\
$[2]$&Eq.~TPR Only   & .532      & (.727/.727)   &  (.299/.339)  & N/A                 \\
$[3]$&*Eq.~TPR + Calibration* & .532      & (.727/.727)   &  (.299/.339)  & .073  \\
\hline
&\emph{Trained on limited features} &   &               &               &                     \\
$[4]$&Accuracy Maximizing& .591     & (.603/.518)   &  (.202/.230)  & .077                \\
\end{tabularx}
\end{table*}

\section{Conclusion}

Decision-makers stand to benefit from algorithmic predictions. This paper studies fair prediction in the widespread setting in which a risk score is constructed to aid their classification tasks. We prove that it is possible to construct calibrated scores that lead to equal error rate classifications at group-blind cutoffs. We characterize exactly when it is possible and propose an algorithm that produces the most accurate score satisfying the fairness criteria and minimizing the decision-maker’s errors. Compared to a commonly practiced strategy of omitting sensitive data, we show that our algorithm can produce scores that enhance both efficiency and equity.


\bibliography{p004-LazarReich}

\appendix

\section{Appendix}
\subsection{Addendum to Theorem \ref{impossibility-result}}
\begin{proof}[Addendum]
We relax condition \eqref{pos-def-cond} of Theorem \ref{impossibility-result} and replace it with the weaker condition that $\Var(Y|Z) > \epsilon$ almost surely. This will correspond to the assumption that $Y$ cannot be perfectly predicted from any realization  of $Z$.

We will make use of the criterion that Borel random variables $R$ and $R'$ are independent conditional on $\Sigma$ iff for all bounded, continuous $f$ and $g$ we have
\[\bb{E}[f(R)g(R')|\Sigma] = \bb{E}[f(R)|\Sigma]\bb{E}[g(R')|\Sigma].\] 
Now suppose that $(Z,A,Y)$ are known to satisfy Theorem \ref{impossibility-result} conditions \eqref{gen-cal} and \eqref{gen-eo}, and that $\Var(Y|Z) > 0$. Then let $\eta$ be a $\mr{Ber}(\varepsilon)$ random variable independent of $(Z,A,Y)$. We consider a variable $A_\eta$ that takes value $A$ with probability $1-\varepsilon$ and otherwise flips the variable $A$ with probability $\varepsilon$, that is,
\begin{equation*}
A_\eta = A + \eta \pmod{2}.
\end{equation*} 
This gives us a triple $(Z, A_\eta, Y)$ that satisfies $\bb{E}[A|Z] \in (0, 1)$ and $\bb{E}[Y|A,Z] \in (0, 1)$ almost surely by construction, corresponding to condition \eqref{pos-def-cond} from the Theorem. We can also show that the triple satisfies the other two conditions. For instance, to show that condition \eqref{gen-cal} holds, let $S$ be an arbitrary set such that $ S \in \sigma(Z)$. 

We will use the fact that any $\sigma(Z)$-measurable random variable $V$ and any random variable $U$ satisfy $\bb{E}[\bb{E}[U|Z]V] = \bb{E}[UV]$. In particular,
\begin{align*}
\bb{E}\left[\bb{E}[f(A_\eta)g(Y)|Z]\mathbbm{1}_{Z \in S}\right] 
&= \bb{E}\left[f(A_\eta)g(Y)\mathbbm{1}_{Z \in S}\right] \annot{$S \in \sigma(Z)$}  \\
&= \bb{E}_{(A,Z,Y)}\left[\bb{E}_\eta[f(A_\eta)]g(Y)\mathbbm{1}_{Z \in S} \right] \annot{$\eta \indep (A,Y,Z)$} \\
&= \bb{E}_{(A,Z,Y)}\left[\bb{E}[\bb{E}_\eta[f(A_\eta)]g(Y)|Z]\mathbbm{1}_{Z \in S} \right] \annot{$S \in \sigma(Z)$}\\
&= \bb{E}_{(A,Z,Y)}\left[\bb{E}[\bb{E}_\eta[f(A_\eta)]|Z]\bb{E}[g(Y)|Z]\mathbbm{1}_{Z \in S} \right] \annot{$Y \indep A \;|\; Z$}\\\
&= \bb{E}_{(A,Z,Y)}\left[\bb{E}_\eta[f(A_\eta)]\bb{E}[g(Y)|Z]\mathbbm{1}_{Z \in S} \right] \\
&= \bb{E}\left[f(A_\eta)\bb{E}[g(Y)|Z]\mathbbm{1}_{Z \in S} \right] \annot{$\eta \indep (A,Y,Z)$} \\
&= \bb{E}\left[\bb{E}[f(A_\eta)|Z]\bb{E}[g(Y)|Z]\mathbbm{1}_{Z \in S} \right],
\end{align*}
where the last step follows since $\bb{E}[g(Y)|Z]\mathbbm{1}_{Z \in S}$ is $Z$-measurable.

\pagebreak

Because $S$ was arbitrary and both $\bb{E}[f(A_\eta)|Z]\bb{E}[g(Y)|Z]$ and $\bb{E}[f(A_\eta)g(Y)|Z]$ are $\sigma(Z)$-measurable, we can conclude that $\bb{E}[f(A_\eta)|Z]\bb{E}[g(Y)|Z] = \bb{E}[f(A_\eta)g(Y)|Z]$ almost surely so \eqref{gen-cal} is satisfied. A very similar argument shows that \eqref{gen-eo} holds. Therefore, by Theorem \ref{impossibility-result}, $A_\eta$ is independent of $(Z, Y)$. Then given arbitrary bounded and continuous functions $f$ and $g$,
\[\bb{E}[f(A_\eta)g(Z,Y)] = \bb{E}[f(A_\eta)]\bb{E}[g(Z,Y)].\] 
Using the fact that $A_\eta \to A$ as $\eta \downarrow 0$ in $L^2(\bb P)$, and that $h \mapsto \bb{E}[h]$ and $(h,h') \mapsto \bb{E}[hh']$ are continuous in $L^2(\bb P)$, we conclude by continuity that 
\[\bb{E}[f(A)g(Z,Y)] = \bb{E}[f(A)]\bb{E}[g(Z,Y)].\] Since $f$ and $g$ were arbitrary, we have in fact shown that $A$ is independent of $(Z,Y)$, as wanted.

Thus, we have succeeded in proving the following refinement: under Theorem \ref{impossibility-result} assumptions \eqref{gen-cal} and \eqref{gen-eo}, if $Y$ cannot be perfectly predicted from any realization  of $Z$, then the random variables $A$ and $(Y,Z)$ must be independent. 

Since assumptions \eqref{gen-cal} and \eqref{gen-eo} continue to hold if we condition on $Z \in S$ for any $S$, we can say further that if Theorem \ref{impossibility-result} conditions \eqref{gen-cal} and \eqref{gen-eo} hold and $P$ is the set of values of $Z$ from which perfect prediction is not possible, i.e. 
$\Var(Y|Z) > 0$ then $A$ and $Y$ are independent conditionally on $Z \in P$.
\end{proof}

\subsection{Proof of Proposition \ref{roc-feas-reg}}
\begin{proof}
First we can prove a lemma stating that $S(A)$ is convex. To see this, let $\xi$ be an independent $\mr{Ber}(\lambda)$ random variable. Then, by iterating expectations, one sees that
\[\alpha(\hat y + \xi(\hat z- \hat y), A) = \lambda \alpha(\hat z,A) + (1-\lambda)\alpha(\hat y, A).\]  
Using this convexity, we can prove the proposition. Note that the points $\alpha(\mathbbm{1}\{p^* \ge c\},A)$ that make up the group-$A$ ROC curve of $p^*$ describe the error rates achieved by all cutoff classifiers based on $p^*$, and so they are in $S(A)$. Meanwhile, since 
\begin{equation*}
\alpha(1 - \hat y, A) = (1, 1) - \alpha(\hat y, A),
\end{equation*}
the points $(1,1) - \alpha(\mathbbm{1}\{ p^* \ge c\},A)$ must also be in $S(A)$. This corresponds to the group-$A$ ROC curve of the scores $1-p^*$. Any point in the convex hull of these two ROC curves can be achieved by randomization as in the aforementioned lemma. For further details and intuition, see Section 4 in Hardt {et al.} \cite{hardt_equality_2016}. Note that Hardt {et al.} choose not to illustrate the feasible region below the main diagonal as it corresponds to classifiers that are worse than random. 

To show that \emph{all} attainable error rates belong to this set, we use the convexity of $S(A)$ to note that the support points of $S(A)$ correspond to all classifiers that yield extrema of $ \gamma_1 \alpha_1(\hat y, A) +\gamma_2 \alpha_2(\hat y, A)$ where $(\gamma_1,\gamma_2)$ are arbitrary weights. To describe these support points tractably, we can use the result derived later in the appendix (Proposition \ref{post-processing-is-sufficient}) that shows that optimal classifications can be chosen to depend on only $p^*$ and $A$, where $p^* = \bb{E}[Y|X,A]$. Thus the extrema of $\gamma \cdot \alpha(\hat y,A)$ are achieved by cutoff rules $f(p^*,A)=\mathbbm{1}\{p^* \ge c\}$ and $f(p^*,A)=\mathbbm{1}\{p^* < c\}$, giving support points 
\[\bigcup_{c \in [0,1]} \bigg\{\alpha(\mathbbm{1}\{p^* \ge c\},A),\; (1,1) - \alpha(\mathbbm{1}\{p^* \ge c\},A)\bigg\},\]
which as we have shown are contained in $S(A)$. Finally, we use the fact that a convex set containing all of its support points is equal to the convex hull of its support points. 
\end{proof} 

\subsection{Proof of Extension of Proposition \ref{ppv-npv-prelim}}
\begin{proof}
Suppose that \ref{cond-equiv-i} holds and call $\hat p_f$ the fair score for which $\hat y =\mathbbm{1}\{\hat p_f \ge \bar p\} $ satisfies equal error rates. Then since $\hat p_f$ is calibrated,
\begin{align*}
\mathbb{P} (Y=1 | \hat y =1,A) = \mathbb{E} [Y| \hat p_f \ge \bar p, A]
= \bb{E}[\hat p_f| \hat p_f &\ge \bar p, A] 
\ge \bar p, \\
\mathbb{P} (Y=1 | \hat y =0,A) &< \bar p.
\end{align*}
So in addition to satisfying equal error rates, $\hat y$ satisfies (\ref{ppv-cond}) and (\ref{npv-cond}), which are equivalent to the two conditions in \eqref{npv-ppv-prelim}. Thus  \ref{cond-equiv-ii} is a necessary condition for fairness. 

Now we show the converse; \ref{cond-equiv-ii} is also sufficient for fairness. Suppose that \ref{cond-equiv-ii} holds and let $\hat y_f$ be a classifier satisfying equal error rates and \eqref{npv-ppv-prelim}. Choose $\hat p(\hat y_f, A) = \mathbb{P}(Y=1 | \hat y_f, A)$. These scores are calibrated by construction. Also, since they satisfy $\hat p(\hat y_f = 0, A) < \bar p$ and $\hat p(\hat y_f = 1, A) \ge \bar p$, they exactly implement the classifier $\hat y_f$ at the cutoff $\bar p$.  
\end{proof}

\subsection{Proof of Theorem \ref{characterizaton}}

\begin{proof}
Building on the above extension of Proposition \ref{ppv-npv-prelim}, it is enough for us to show that the existence of the point $(\alpha_1,\alpha_2) \in S(L) \cap S(H)$ satisfying (\ref{npv-ppv-thm}) is equivalent to the following: There exists a classifier $\hat y$ satisfying equal error rates and (\ref{npv-ppv-prelim}).

First note that $S(L) \cap S(H)$ is nonempty, since for example $(0,0)$ and $(1,1)$ are points in both $S(L)$ and $S(H)$. So we can consider some arbitrary $(\alpha_1,\alpha_2)$ that is in $S(L) \cap S(H)$ and is therefore implementable by an equal error rate classifier that we call $\hat y_e$. We need to show that $\hat y_e$ satisfying the conditions in (\ref{npv-ppv-prelim}) $\forall A$ is equivalent to its corresponding true and false positive rates $(\alpha_1(\hat y_e, A),\alpha_2(\hat y_e, A))$ satisfying (\ref{npv-ppv-thm}) $\forall A$. 
Recall that the PPV condition in \eqref{npv-ppv-prelim} required
\[\bb{P}(Y=1|\hat y_e = 1, A) \ge \bar p.\]
Applying Bayes' rule to the inequality, we have
\begin{align*}
\bb{P}(Y=1|\hat y_e = 1, A) 
&= \frac{\bb{P}(\hat y_e=1| Y = 1, A)\bb{P}(Y = 1| A)} {\bb{P}(\hat y_e = 1 | A)} \\
&= \frac{\alpha_2(\hat y_e, A) \mu_A}{\alpha_2(\hat y_e, A) \mu_A + \alpha_1(\hat y_e, A) (1 - \mu_A)} \ge \bar p.
\end{align*}
After algebraic manipulation, the restriction can be written 
\begin{equation*}
\frac{\alpha_2(\hat y_e, A)}{\alpha_1(\hat y_e, A)} \ge \frac{\bar p (1-\mu_A)}{(1-\bar p)\mu_A} = \frac{\bar p }{(1-\bar p)\beta_A}.
\end{equation*}
where $\beta_A \equiv \nf{\mu_A}{(1-\mu_A)}$. Therefore $(\alpha_1(\hat y_e, A),\alpha_2(\hat y_e, A))$ must satisfy the following 
\begin{equation*}
\frac{\alpha_2(\hat y_e, A)}{\alpha_1(\hat y_e, A)} \ge \frac{\bar p }{(1-\bar p)\beta_A}
\end{equation*}
Since $\beta_L < \beta_H$, the condition is more restrictive when $A=L$, giving the first condition in (\ref{npv-ppv-thm}). We next similarly transform the NPV condition in (\ref{npv-ppv-prelim}), recalling it requires  
$\bb{P}(Y=0|\hat y = 0, A) > 1- \bar p $. 
By Bayes' rule,
\begin{align*}
\bb{P}(Y=0|\hat y = 0, A)
&= \frac{\bb{P}(\hat y=0| Y = 0, A)\bb{P}(Y = 0| A)} {\bb{P}(\hat y = 0 | A)} \\
&= \frac{(1-\alpha_1(\hat y, A)) (1-\mu_A)}{(1-\alpha_1(\hat y, A)) (1-\mu_A) + (1-\alpha_2(\hat y, A))  \mu_A} > 1- \bar p.
\end{align*}
After algebraic manipulation, this becomes $\forall A$
\begin{equation*}
\frac{(1-\alpha_1(\hat y, A))}{(1-\alpha_2(\hat y, A))} >  \frac{(1-\bar p)\beta_A}{\bar p}.
\end{equation*}
Since $\beta_H > \beta_L$, the most restrictive case is when $A=H$, giving the second condition in (\ref{npv-ppv-thm}).

Note that special attention should be given to the corner solutions. At point $(0,0)$, the first condition in (\ref{npv-ppv-thm}) becomes irrelevant and so the second condition in (\ref{npv-ppv-thm}) is necessary and sufficient. Meanwhile at $(1,1)$, the second condition in (\ref{npv-ppv-thm}) becomes irrelevant so the first condition in (\ref{npv-ppv-thm}) is necessary and sufficient.
\end{proof} 

\subsection{Proof of Corollary \ref{checkable-cond}}
\begin{proof}

Let $F$ and $G$ denote the lines for which the inequalities \eqref{npv-ppv-thm} hold with equality. That is to say, $F, G \subset \bb{R}^2$ are given by
\begin{equation*}
F = \Set[(\alpha_1,\alpha_2) \in \bb{R}^2]{\frac{\alpha_2}{\alpha_1} = \frac{\bar p}{\beta_L (1 - \bar p)} }, \qquad 
G = \Set[(\alpha_1,\alpha_2) \in \bb{R}^2]{\f {(1-\alpha_1)}{(1-\alpha_2)} =  \f{\beta_H (1-\bar p)}{\bar p}}
\end{equation*}
The lines intersect at $(\breve \alpha_1, \breve \alpha_2)$ given by (\ref{breve}). Our proof will rest on a few basic facts: $S(L) \cap S(H)$ is convex, $F$ contains $(0,0)$, $G$ contains $(1,1)$, and both lines have positive slope. 
First we prove that if $\breve \alpha_1 \le 0$, $\breve \alpha_1 \ge 1$, or both ROC curves lie above the intersection $(\breve \alpha_1, \breve \alpha_2)$, then there exists a point $(\alpha_1, \alpha_2)$ satisfying the feasibility conditions in Theorem \ref{characterizaton}. 

\textit{Case I: $0 < \breve{\alpha_1} < 1$ and $(\breve{\alpha_1}, \breve{\alpha_2})$ lies below both ROC curves.}  Note that increasing $\alpha_2$ slackens both inequalities \eqref{npv-ppv-thm}. Thus, if $0 < \breve{\alpha_1} < 1$ and $(\breve{\alpha_1}, \breve{\alpha_2})$ lies below both ROC curves, there then exists a point $(\breve{\alpha_1}, \alpha_2)$ with $\alpha_2 > \breve \alpha_2$ that lies on the minimum of the two ROC curves, hence in $S(H) \cap S(L)$, and moreover the inequalities \eqref{npv-ppv-thm} hold at $(\breve{\alpha_1}, \alpha_2)$. This is a feasible point. 

\textit{Case II: $\breve \alpha_1 \le 0$. } On the other hand, if $\breve \alpha_1 \le 0$, then in $(0,1) \times \bb R$ the line $F$ lies strictly above $G$. Then the point $(0,0) \in S(L) \cap S(H) \cap F$ lies above $G$, meaning that the second condition in (\ref{npv-ppv-thm}) holds and the point is feasible.  

\textit{Case III: $\breve \alpha_1 \ge 1$. } If $\breve \alpha_1 \ge 1$, then in $(0,1) \times \bb R$ the line $G$ lies strictly above $F$. Then the point $(1,1) \in S(L) \cap S(H) \cap G$ lies above $F$, so the first condition in (\ref{npv-ppv-thm}) holds and the point is feasible.

Finally, we prove the converse that if $0 < \breve \alpha_1 < 1$ and $\breve \alpha_2$ lies above at least one of the ROC curves, then the feasible region is empty. Let the intersection of  $S(L) \cap S(H)$ with the half-space above $F$ be denoted by $I_F$, and the intersection of  $S(L) \cap S(H)$ with the half-space above $G$ be denoted by $I_G$. We need to show that $I_F \cap I_G$ is empty.  The argument follows from the convexity of  $S(L) \cap S(H)$ and the fact that both $F$ and $G$ have positive slopes. In particular, due to the convexity of  $S(L) \cap S(H)$, the positive slope of $F$, and the fact that $(0,0)$ is in $F$, we know the line $F$ must intersect the boundary of  $S(L) \cap S(H)$ strictly to the left of $\breve \alpha_1$. Meanwhile, $G$ must intersect the boundary of  $S(L) \cap S(H)$ strictly to the right of $\breve \alpha_1$. Thus the rightmost point of $I_F$ lies strictly to the left of the leftmost point of $I_G$, and the intersection of $S(L) \cap S(H)$ with both half-spaces above $F$ and $G$ must be empty.  \end{proof} 

\subsection{Justification for post-processing \texorpdfstring{\boldmath $p^*$}{p*} in algorithm}

First we justify post-processing the Bayes optimal $p^*$ to arrive at the optimal fair $\hat p$. To do so we adapt Proposition 5.2 from Hardt {et al.} \cite{hardt_equality_2016} to our setting and prove the following

\enlargethispage{-1\baselineskip}
\begin{proposition}\label{post-processing-is-sufficient}
For any source distribution over $(Y, X, A)$ with Bayes optimal regressor given by $p^*(X,A)=\bb{E}[Y|X,A]$ and  loss function $\ell$, there exists a predictor $\hat p (p^*,A)$ such that

\begin{romanenumerate}
\item $\hat p$ is an optimal predictor satisfying our fairness properties of calibration and equal error rates. That is, $\bb{E}[\ell(\mathbbm{1}_{\hat p >\underline{p}}, Y)] \leq \bb{E}[\ell(\mathbbm{1}_{\hat{g}>\underline{p}}, Y)]$ for any $\hat g$ that satisfies the properties. 
\item $\hat p$ is derived from $(p^*,A)$. In particular, it is a (possibly random) function of the random variables $(p^*,A)$ alone, and is independent of $X$ conditional on $(p^*,A)$.  
\end{romanenumerate}
\end{proposition}
\begin{proof}
To start, first note that our fairness properties of calibration and equal error rates on a score $p$ and classifications $\mathbbm{1}\{ p \ge \bar p\}$ are ``oblivious.'' That is, they depend only on the joint distribution of $(Y, A, p)$ given the known cutoff $\bar p$. We will show that for any arbitrary $\hat g$ that satisfies the fairness properties, we can construct a $\hat p$ that also satisfies fairness, yields the same expected loss, and is derived from $(p^*,A)$. 

Consider an arbitrary $\hat g =f(X, A)$ satisfying the fairness properties. We can define $\hat p (p^*, A)$ as follows:  draw a vector $X'$ independently from the conditional distribution of $X$ given the realized values of $p^*$ and $A$, and set $\hat p = f (X', A)$. Note this $\hat p $ satisfies (ii) by construction.     

To show that this $\hat p$ satisfies the fairness properties and yields the same expected loss as $\hat g$, note that since $Y$ is binary with conditional expectation equal to the Bayes optimal $p^*$, we know $Y$ is independent of $X$ conditional on $p^*$. Therefore $(Y, p^*, X, A)$ and $(Y, p^*, X', A)$ have the same joint distribution, and so must $(f(X,A),A,Y)$ and $(f(X',A),A,Y)$. Since the fairness properties are oblivious and depend only on these latter joint distributions, then we know that as long as $\hat g$ satisfies them then so will $\hat p$. Finally, we can deduce that $(Y, \hat g)$ and $(Y, \hat p)$ also have the same joint distribution, meaning that (i) is satisfied with equality. 
\end{proof}

\subsection{Our algorithm as a mean-preserving contraction of scores}

We observe that a calibrated score derived from another is a mean-preserving contraction. Since the Bayes optimal $p^*$ that serves as input to our algorithm frequently satisfies calibration (see Liu {et al.} 2019), then our post-processing method can be viewed as finding its smallest mean preserving contraction that achieves equal error rates at the decision-maker's cutoff. 

The relationship between calibrated scores related by post-processing is characterized by our proposition below.  

\begin{proposition}
Let $p_A$ be any calibrated score of group $A$, i.e. satisfying $\bb{E}[Y|p_A] =p_A$ for members of $A$, and let $\hat{p}_A = f(p_A, \zeta)$ be a score post-processed from $p_A$ that is also calibrated, where $\zeta$ is independent of $Y$ conditional on $p_A$. 
Then, $\hat{p}_A$ is a mean-preserving contraction of $p_A$, with $p_A = \hat{p}_A + Z $ and $\bb{E}[Z|\hat{p}_A]=0$. Conversely, any $\tilde{p}_A$ that satisfies $p_A = \tilde{p}_A + Z $ with $\bb{E}[Z|\tilde{p}_A]=0$ is calibrated. 
\end{proposition}
\begin{proof}
We first show that $\hat{p}_A $ is a mean-preserving contraction of $p_A $. To start, note that the post-processed $\hat{p}_A$ is assumed to be calibrated, so $\bb{E}[Y|\hat{p}_A] =\hat{p}_A$. Moreover, since $\hat{p}_A = f(p_A, \zeta)$, we have $\sigma(\hat{p}_A) \subseteq \sigma(p_A, \zeta) $. Therefore by the tower property of conditional expectation,
\begin{align*}
\hat{p}_A = \bb{E} [ Y | \hat{p}_A ] 
&=  \bb{E}[ \bb{E}[Y|p_A, \zeta ] | \hat{p}_A]  \\
&= \bb{E}[ \bb{E}[Y|p_ A ] | \hat{p}_A]  \annot{by conditional independence of $\zeta$} \\
&= \bb{E}[p_A |\hat{p}_A] \annot{by calibration of $p_A$}.
\end{align*}
Then $p_A = p_A + (\hat{p}_A - \bb{E}[p_A |\hat{p}_A]) = \hat{p}_A + (p_A - \bb{E}[p_A |\hat{p}_A])$ where the second term is by construction mean independent of $\hat{p}_A$, so $\hat{p}_A$ is a mean-preserving contraction of $p_A$. 

Now we show that if the score $\tilde{p}_A$ is a mean-preserving contraction of $p_A$ such that $p_A = \tilde{p}_A + Z$ for some $Z$ satisfying $\bb{E}(Z| \tilde{p}_A) = 0$, then $\tilde{p}_A$ is calibrated. Observe that  
\begin{align*}
\bb{E}[p_A | \tilde{p}_A ] &= \bb{E}[\tilde{p}_A + Z | \tilde{p}_A ]  = \bb{E}[ \tilde{p}_A | \tilde{p}_A ] + \bb{E}[ Z | \tilde{p}_A ] = \tilde{p}_A
\end{align*}
which is sufficient to show that $\tilde{P}_A$ is calibrated. To see why, recall that $p_A$ is calibrated and note that by the tower property of conditional expectation with $\sigma(\tilde{p}_A) \subseteq \sigma(p_A) $,
\begin{align*}
\bb{E}[p_A | \tilde{p}_A ] &= \bb{E}[ \bb{E}(Y | p_A) | \tilde{p}_A ] = \bb{E}[ Y | \tilde{p}_A ].
\qedhere
\end{align*}
\end{proof}

\subsection{Justification for discretizing \texorpdfstring{\boldmath $p^*$}{p*}}

Our algorithm uses the discretization of $p^*$ to construct a linear program that maps probability masses from $p^*$ to $\hat p$. Note that even if the original $p^*$ is not discrete, it can easily be discretized into $N$ bins by taking $p' = \lfloor Np^* \rfloor/N.$ The discretized score will satisfy $|p' - p^*| \le N^{-1}$ almost surely, so for large values of $N$, the discretization $p'$ approximates $p^*$ well.

\subsection{Online appendix for credit lending application}
\subsubsection*{Raw data and cleaning}

Our empirical application is based on public data collected and made available by the U.S. Census Bureau, specifically the 2014 Survey of Income and Program Participation  \footnote{https://www.census.gov/programs-surveys/sipp/data/datasets.html}. We converted the datasets from Waves 2 and 4 to CSV format and then organized them to serve our prediction task: use features in Wave 2 to predict reported repayment ability in Wave 4. 

We matched every adult from the Wave 2 survey who responded to the Wave 4 survey and dropped the non-responders. We used education reported in Wave 2 to distinguish two groups $L$ and $H$ (representing 44\% and 56\% of the population respectively),  $L$ who attained at most a high school education and $H$ who attained more. We randomly allocated 30\% of all observations to a test set (about 8,000 adults) and the remaining 70\% to a training set (about 18,000 adults).

Our outcome was the respondents' ability to pay mortgage, rent, and utilities in every month tracked in 2016 according to Wave 4. Any adult who failed to pay mortgage, rent, and/or utilities in any month was assigned label $Y=0$, and otherwise assigned $Y=1$. Base rates differed across groups; 11\% of the less-educated group missed a payment compared to only 7\% of the higher-educated group.

Finally, we constructed two sets of features. The first was based on rich data, comprising virtually all available variables from the Wave 2 survey but dropping those with no variation in the training set (leaving over 2,000 in total). The second was based on limited data, where we hand-selected ``non-sensitive'' variables involving assets,  debts, income, and employment (over 800 in total). 

We identified which features were categorical and performed one-hot encoding. Then we standardized all features by centering them at 0 and dividing by their feature-specific standard deviations from the training set. 

\renewcommand\thefigure{A.\arabic{figure}}
\begin{figure}[h!] 
\begin{subfigure}[t]{.5\textwidth}
  \includegraphics[height=2in]{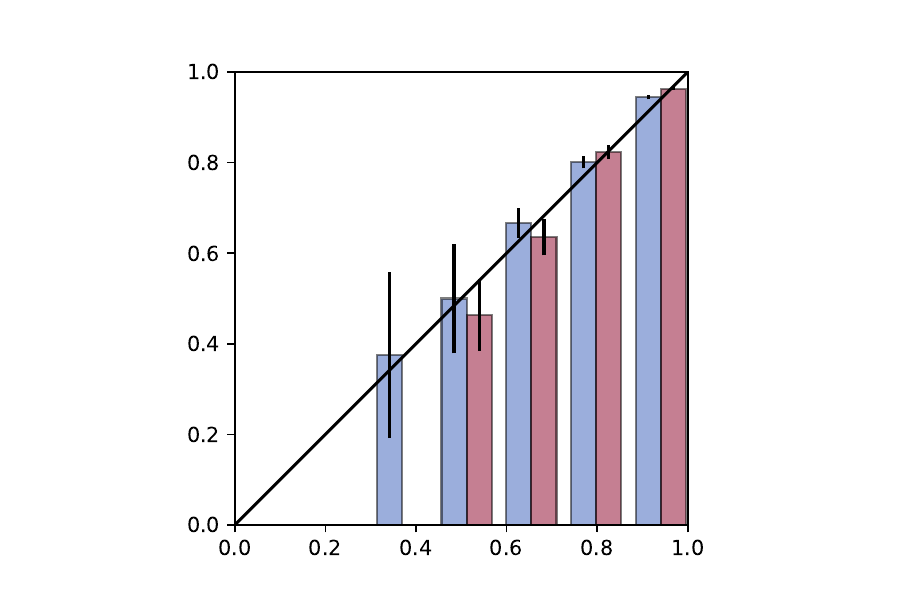}
  \caption{}\label{sipp-cal-pre}
\end{subfigure}
\hspace{.1in}
\begin{subfigure}[t]{.5\textwidth}
  \includegraphics[height=2in]{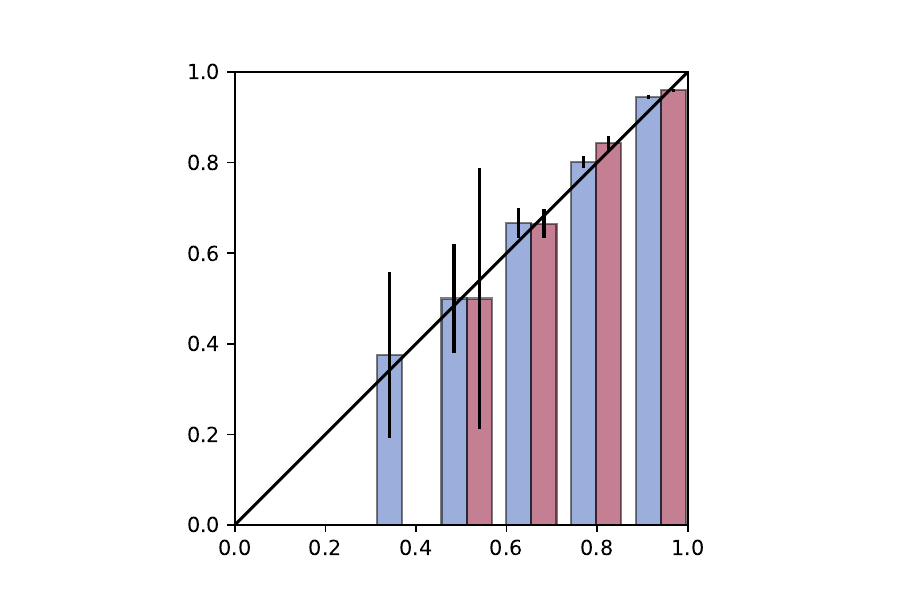}
  \caption{}\label{sipp-cal-post}
\end{subfigure}
\caption{Credit lending calibration plots. Panel \ref{sipp-cal-pre} depicts discretized pre-processed scores on the horizontal axis, with the portion in each bin paying their bills plotted on the vertical axis (including standard errors). Panel \ref{sipp-cal-post} depicts the calibration of post-processed scores. Our procedure is seen to preserve calibration.}
\label{lending-calibration}
\end{figure}

\begin{figure} 
\begin{subfigure}{1\textwidth}
\centering
 \includegraphics[trim=0 50 0 50, clip, height=2in]{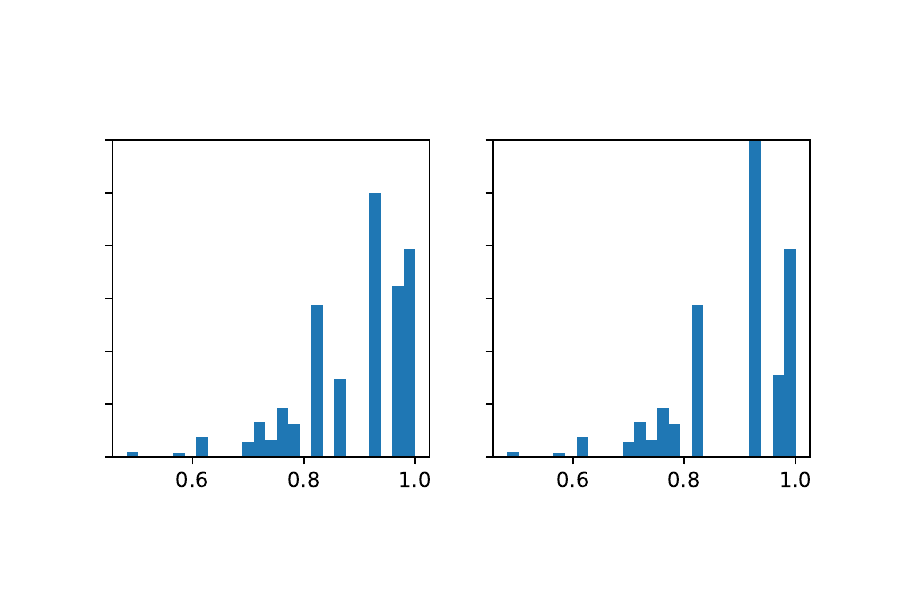}
\label{sipp-score-change-low}
\end{subfigure}
\begin{subfigure}{1\textwidth}
\centering
  \includegraphics[trim=0 50 0 50, clip, height=2in]{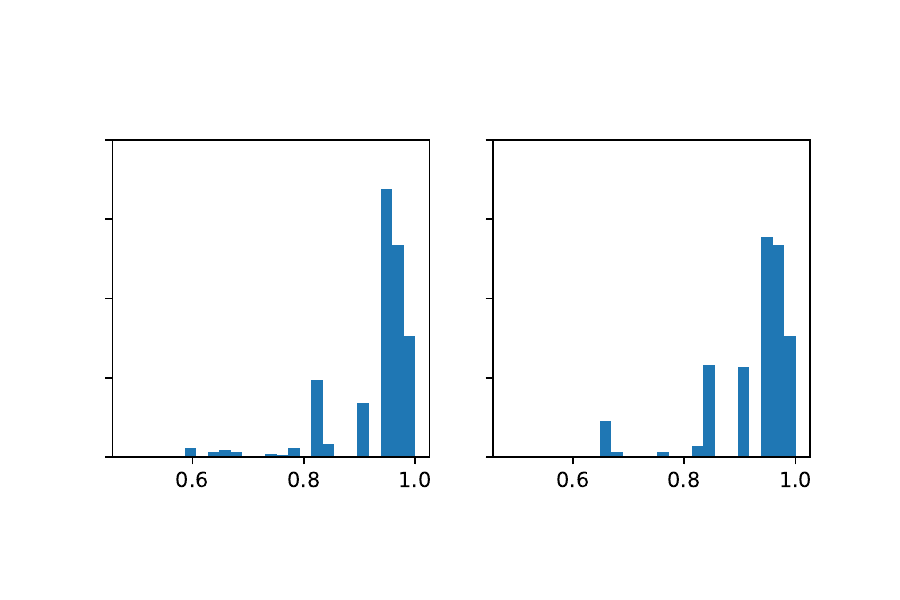}
\label{sipp-score-change-high}
\end{subfigure}
\begin{subfigure}{1\textwidth}
\centering
  \includegraphics[trim=0 50 0 50, clip, height=2in]{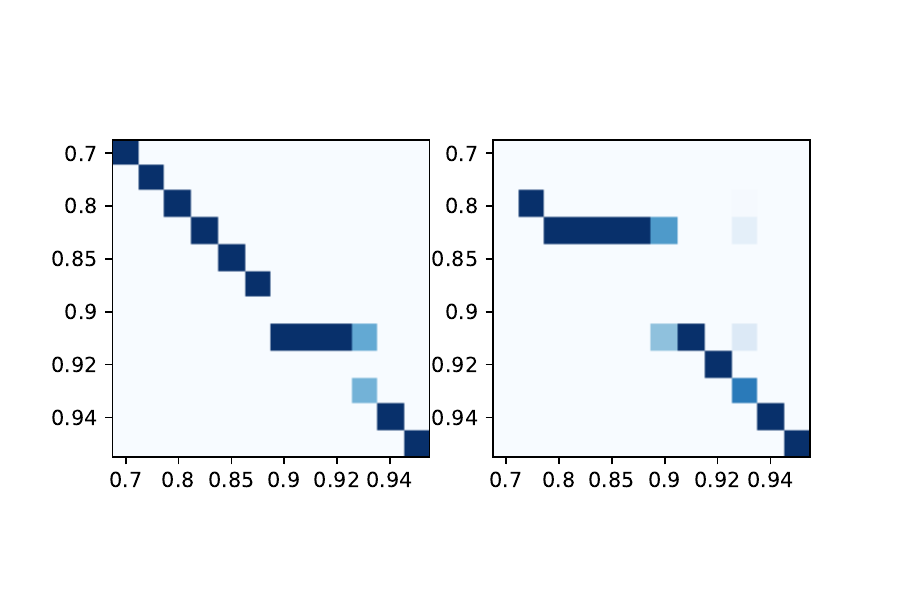}
\label{sipp-kernel}
\end{subfigure}
\caption{Credit lending score comparisons. The top-most plots depict the distribution of scores in the less-educated group (inputted $p^*$ on the left and outputted $\hat p$ on the right). The middle plots depict the distribution of scores in the high-educated group (inputted $p^*$ on the left and outputted $\hat p$ on the right). The bottom plots depict how the post-processing procedure assigns probability masses from the inputted score (horizontal axis) to the outputted score (vertical axis), with the less-educated group's transformation depicted to the left and the high-educated group's transformation depicted to the right.} 
\label{lending-score-change}
\end{figure}

\subsubsection*{Deriving our empirical results}

In deriving our empirical results, we employed the following protocol. As an initial step we estimated the original scores $p^*$ with LASSO where the penalty parameters were tuned using 10-fold cross validation in the training dataset. We then tuned and evaluated the post-processing procedure in the test dataset. This consisted of the following steps.

\begin{enumerate}

\item We compute a discrete approximation to the score distribution of $p^*$ for each group using the \texttt{numpy.histogram} Python method. This involves setting the user-defined hyperparameter $N$ for the number of bins. We produced all results with the specification $N = 50$. We also tried $N = 10, 15, 25, 100, 250$, which did not appear to change results significantly. For $N = 500, 1000$, results also did not change significantly but the running time was significantly longer.

\item Next we calibrate the discrete approximation to the data by replacing the score assigned to each bin with the average outcome. 
\color{black}
\item We use the calibrated and discretized  scores to compute group-specific ROC curves using the \texttt{{scikitlearn.metrics.roc\textunderscore curve}} function, and then compute the calibration compatibility constraints, assuming $k=10 \implies \bar p = \frac{10}{11}$. These determine the feasible region $R(H) \times R(L)$.  

\item We define the loss (\ref{dual}) as a function of error rates using $k=10$, picking $\Lambda$ to equate the true positive rates across groups. We minimize that loss in the feasible region using the \texttt{cvxpy} convex optimization library. 
\color{black}
We directly report the losses corresponding to the optima found by our procedure. Our comparisons correspond to removing 
\begin{enumerate}
\item the calibration compatibility constraints,
\item the calibration compatibility constraints and the \textcolor{black}{equal opportunity constraint}
\item both constraints, as well as omitting sensitive features from estimation of $p^*$. 
\end{enumerate}

\item Then we use our post-processing method to back out the most informative score $\hat p$ that produces the optimal error rates. In particular, we compute the transformation kernel $T$ using the \texttt{cvxpy} convex optimization library. Next we output post-processed scores by randomly mapping individuals' original scores given by $p^*$ to new scores $\hat p$ with probabilities specified by the kernel $T$.

\end{enumerate}

Finally, we evaluate our performance by inspecting the calibration of the output score and computing MSEs, error rates, and the decision-maker's loss. All losses we report are computed as a function of error rates, according to
\begin{align*}
\bb{E}\ell(\hat y, Y) 
&= \bb{E}[k\mathbbm{1}\{\hat y > Y\} + \mathbbm{1}\{\hat y < Y\}] \\
&= k\bb{P}(Y=0)\bb{P}(\hat y = 1 | Y = 0) + \bb{P}(Y=1)\bb{P}(\hat y = 0 | Y = 1).
\end{align*}
We simply replace the conditional probabilities by empirical averages from the test dataset. For the MSE of a risk score $\hat p$, we report $\bb{E}_n[(Y - \hat p)^2]$ as is standard. Although not reported in the table, the standard deviation of the MSE of our (randomized) post-processed scores from 100 repetitions is 0.0001. 

To assess the extent to which our post-processing preserves calibration, in Figure \ref{lending-calibration} we plotted score bins on the horizontal axis and the average outcomes within each bin along the vertical axis. Error bars depict the standard error of the mean estimate within each bin. 

We can also study how the post-processing transforms the most accurate estimates of $p^*$ to the outputted scores $\hat p$ that satisfy the fairness criteria, Figure \ref{lending-score-change} depicts in detail how the post-processing procedure shifts the original distribution of scores. 

\subsection{Online appendix for criminal justice application}

\subsubsection*{Raw data and cleaning}

The second example in our paper shows that our procedure can modify existing risk assessments to output calibrated scores and corresponding binary summaries satisfying equal error rates. We used the Broward County dataset of COMPAS risk scores made available by \textit{ProPublica}. \footnote{The file ``compas-scores-two-years.csv'' is available at https://github.com/propublica/compas-analysis}

Motivated by \textit{ProPublica}'s analysis, we chose as our outcome the variable ``two\_year\_recid'' and supposed that COMPAS scores from 1-4 are classified as low risk while those from 5-10 are classified as high risk. We also considered only defendants labelled as white and black (40\% and 60\% respectively from a total sample of 6,150). Their recidivism rates vary. A percentage 51\% of black defendants recidivated within two years, compared to 39\% of the white defendants.  

We define a positive label $Y=1$ as \textit{not} recidivating within two years, and otherwise assign label $Y=0$. Defined as such, the white defendants in the dataset have a higher base rate than the black defendants. 

\begin{figure} 
\begin{subfigure}{1\textwidth}
\centering
 \includegraphics[trim=0 50 0 50, clip, height=2in]{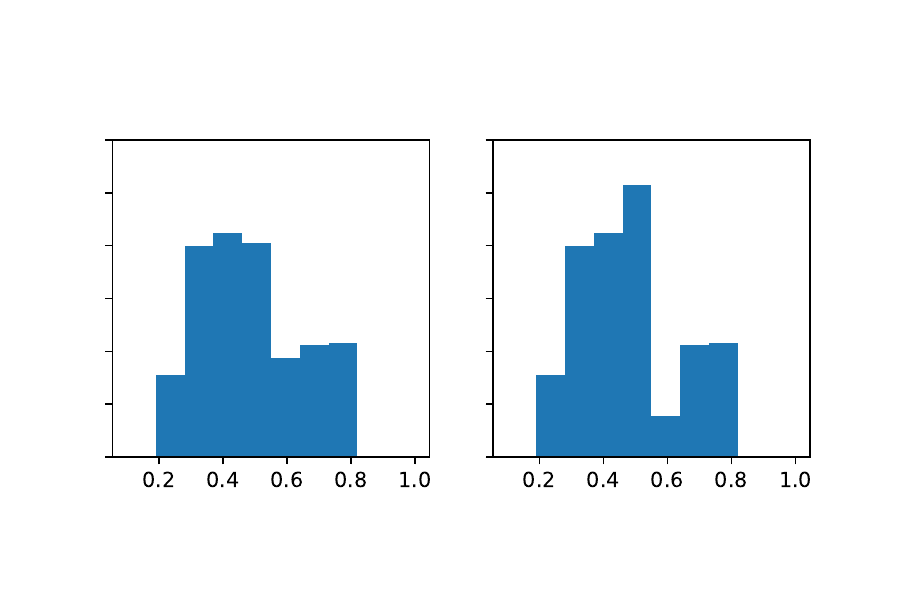}
\label{compas-score-change-black}
\end{subfigure}
\begin{subfigure}{1\textwidth}
\centering
  \includegraphics[trim=0 50 0 50, clip, height=2in]{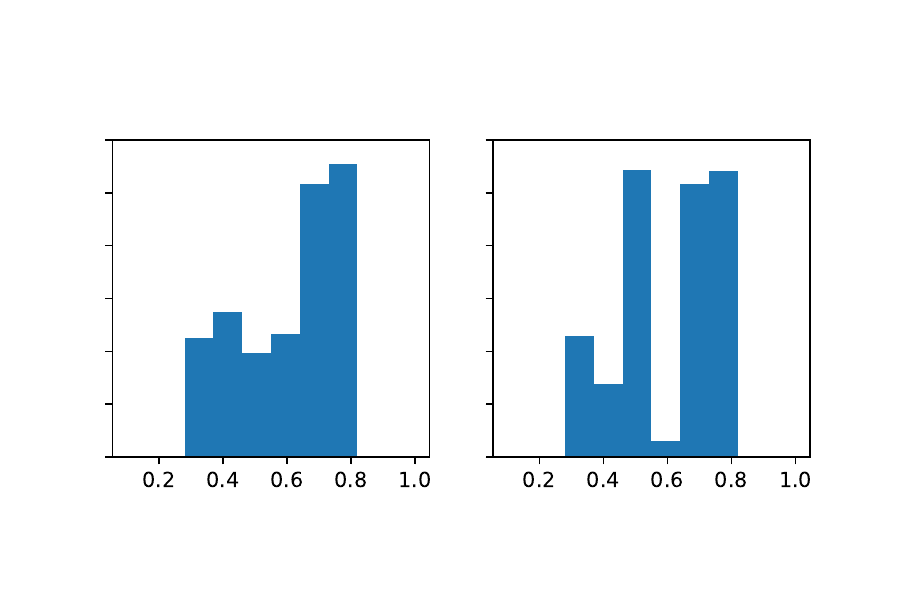}
\label{compas-score-change-white}
\end{subfigure}
\begin{subfigure}{1\textwidth}
\centering
  \includegraphics[trim=0 50 0 50, clip, height=2in]{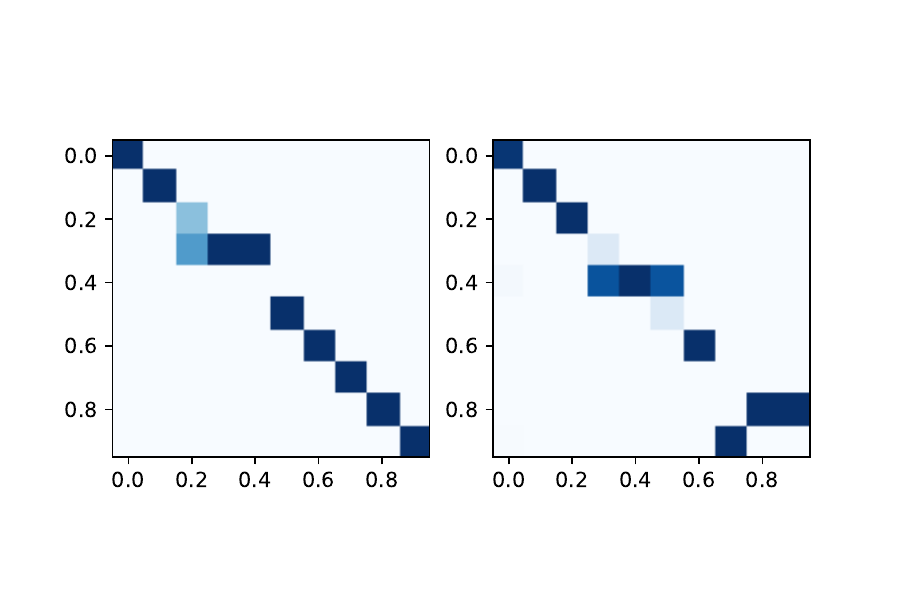}
\label{compas-kernel}
\end{subfigure}
\caption{Criminal justice score comparisons. Recall that we defined the scores to signify probabilities of \textit{not} recidivating. The top-most plots depict the distribution of scores among black defendants in the dataset (inputted $p^*$ on the left and outputted $\hat p$ on the right). The middle plots depict the distribution of scores among white defendants in the dataset (inputted $p^*$ on the left and outputted $\hat p$ on the right). The bottom plots depict how the post-processing procedure assigns probability masses from the inputted score (horizontal axis) to the outputted score (vertical axis), with the black defendants' transformation kernel depicted to the left and the white defendants' transformation kernel depicted to the right.} 
\label{compas-scores}
\end{figure}

\subsubsection*{Deriving our empirical results}

Our results followed these steps.

\begin{enumerate}

\item We divide all defendants' given decile scores by 10 so they lie between 0 and 1. 

\item Next we calibrate the group-specific scores by replacing each with the average outcome of individuals assigned that score. 

\item We use the calibrated discrete scores to compute group-specific ROC curves using the \texttt{{scikitlearn.metrics.roc\textunderscore curve}} function. 

\item Then we back out the effective $k$ and $\bar p$ that serve as inputs to the calibration compatibility constraint and the loss function. In particular, we wish to maintain the same effective risk cutoff in our post-processing as used in the \textit{ProPublica} analysis on original COMPAS scores. Therefore for our purposes we define the cutoff $\bar p$ to be the minimum score (after calibration) that was classified in the \textit{ProPublica} analysis as ``high.'' Along with the group base rates, this determines the calibration compatibility constraints that combined with our ROC curves give the feasible region. Given the corresponding $k$, we define the loss function and find the optimal target rates in that region. 

\item We use our risk score optimization method to back out the most informative scores that produce the target error rates. In particular, we compute the transformation kernel $T$ using the \texttt{cvxpy} convex optimization library. Then we output post-processed scores by randomly mapping individuals' original scores given by $p^*$ to new scores $\hat p$ with probabilities specified by the kernel $T$.

\end{enumerate}

Finally, we plot the error rates that our post-processing achieves and compare them to the disparate rates found by \textit{ProPublica}. We also produce a calibration plot showing that our procedure preserves predictive parity of the scores. To supplement the plots from the paper, Figure \ref{compas-scores} depicts how the post-processing procedure shifts the original distribution of scores to achieve the fairness criteria.

\end{document}